%% file: sample_paper.tex
\documentclass[twoside]{article}

\usepackage[accepted]{aistats2025}
\usepackage{amsmath}
\usepackage{amssymb}
\usepackage{mathtools}
\usepackage{amsthm}
\usepackage{rotating}
\usepackage{tablefootnote}
\DeclareMathOperator*{\argmin}{arg\,min}

\DeclareMathOperator{\cumsum}{cumsum}
\DeclareMathOperator{\card}{card}

\theoremstyle{plain}
\newtheorem{theorem}{Theorem}[section]
\newtheorem{proposition}[theorem]{Proposition}
\newtheorem{lemma}[theorem]{Lemma}

\theoremstyle{definition}

\theoremstyle{remark}

\usepackage{algorithm}
\usepackage{algorithmic}
\usepackage{caption}
\usepackage{wrapfig}
\usepackage{subcaption}
\usepackage{graphicx}

\let\emptyset\varnothing
\usepackage[export]{adjustbox}

\usepackage[bottom]{footmisc}
\usepackage{booktabs}
\usepackage{url} 
\usepackage{hyperref}

\usepackage[round]{natbib}

\begin{document}

%

%

\twocolumn[

\aistatstitle{Strong Screening Rules for Group-based SLOPE Models}

\aistatsauthor{Fabio Feser \And Marina Evangelou}
\aistatsaddress{ Imperial College London \And  Imperial College London} ]

\begin{abstract}
Tuning the regularization parameter in penalized regression models is an expensive task, requiring multiple models to be fit along a path of parameters. Strong screening rules drastically reduce computational costs by lowering the dimensionality of the input prior to fitting. We develop strong screening rules for group-based Sorted L-One Penalized Estimation (SLOPE) models: Group SLOPE and Sparse-group SLOPE. The developed rules are applicable to the wider family of group-based OWL models, including OSCAR. Our experiments on both synthetic and real data show that the screening rules significantly accelerate the fitting process. The screening rules make it accessible for group SLOPE and sparse-group SLOPE to be applied to high-dimensional datasets, particularly those encountered in genetics.
\end{abstract}

\section{INTRODUCTION}
\input{new_sections/1-introduction_v2}

\section{SPARSE-GROUP STRONG SCREENING}\label{sec:framework}
\input{new_sections/2-sparse-group-framework}

\section{GROUP SLOPE}\label{section:gslope}
\input{new_sections/3-gslope}

\section{SPARSE-GROUP SLOPE}\label{section:sgs}
\input{new_sections/4-sgs}

\section{RESULTS}\label{section:results}
\input{new_sections/5-synthetic}

\section{DISCUSSION}\label{section:discussion}
\input{new_sections/6-discussion}

\bibliography{new_ref_v2}
\bibliographystyle{plainnat}

\section*{Checklist}

 \begin{enumerate}

 \item For all models and algorithms presented, check if you include:
 \begin{enumerate}
   \item A clear description of the mathematical setting, assumptions, algorithm, and/or model. [Yes, we have defined our approach and the problem setting.]
   \item An analysis of the properties and complexity (time, space, sample size) of any algorithm. [Yes, we have described the complexity of the fitting algorithm.]
   \item (Optional) Anonymized source code, with specification of all dependencies, including external libraries. [Yes]
 \end{enumerate}

 \item For any theoretical claim, check if you include:
 \begin{enumerate}
   \item Statements of the full set of assumptions of all theoretical results. [Yes. We have presented our assumptions in the results and have discussed the limitations of these in the discussion.]
   \item Complete proofs of all theoretical results. [Yes. Proofs have been provided in the Appendix for all of our results.]
   \item Clear explanations of any assumptions. [Yes, we have extensively discussed our assumptions and any corresponding limitations.]     
 \end{enumerate}

 \item For all figures and tables that present empirical results, check if you include:
 \begin{enumerate}
   \item The code, data, and instructions needed to reproduce the main experimental results (either in the supplemental material or as a URL). [Yes, code has been provided to reproduce any figure.]
   \item All the training details (e.g., data splits, hyperparameters, how they were chosen). [Yes, we have given the full simulation set up in Table \ref{tbl:atos_params}.]
         \item A clear definition of the specific measure or statistics and error bars (e.g., with respect to the random seed after running experiments multiple times). [Yes, we have described our statistical measures and our error bars.]
         \item A description of the computing infrastructure used. (e.g., type of GPUs, internal cluster, or cloud provider). [Yes]
 \end{enumerate}

 \item If you are using existing assets (e.g., code, data, models) or curating/releasing new assets, check if you include:
 \begin{enumerate}
   \item Citations of the creator If your work uses existing assets. [Yes, we have cited our data and model sources.]
   \item The license information of the assets, if applicable. [Yes]
   \item New assets either in the supplemental material or as a URL, if applicable. [Not Applicable]
   \item Information about consent from data providers/curators. [Yes, this is included in the code supplementary materials.]
   \item Discussion of sensible content if applicable, e.g., personally identifiable information or offensive content. [Not Applicable]
 \end{enumerate}

 \item If you used crowdsourcing or conducted research with human subjects, check if you include:
 \begin{enumerate}
   \item The full text of instructions given to participants and screenshots. [Not Applicable]
   \item Descriptions of potential participant risks, with links to Institutional Review Board (IRB) approvals if applicable. [Not Applicable]
   \item The estimated hourly wage paid to participants and the total amount spent on participant compensation. [Not Applicable]
 \end{enumerate}

 \end{enumerate}
 \newpage
\appendix
\onecolumn
\renewcommand\thefigure{A\arabic{figure}}
\setcounter{figure}{0} 
\setcounter{table}{0}
\renewcommand{\thetable}{A\arabic{table}}
\setcounter{algorithm}{0}
\renewcommand{\thealgorithm}{A\arabic{algorithm}}
\aistatstitle{Strong Screening Rules for Group-based SLOPE Models: \\
Supplementary Materials}
\section{GROUP SLOPE}
\subsection{Penalty Weights}\label{appendix:gslope_pen_seq}
The penalty weights for gSLOPE were derived to provide group FDR-control under orthogonal designs \citep{Brzyski2019GroupPredictors}. For the FDR-control parameter $q_g\in(0,1)$, they are given by (where the indexing corresponds to the sorted groups)
\begin{equation*}
w_i^\text{max} = \max_{j=1,\dots,m}\left\{\frac{1}{\sqrt{p_j}} F^{-1}_{\chi_{p_j}} (1-q_gi/m)\right\}, \; \text{for} \; i=1,\dots,m,
 \end{equation*}
where $F_{\chi_{p_j}}$ is the cumulative distribution function of a $\chi$ distribution with $p_j$ degrees of freedom. A relaxtion to this sequence is applied in \cite{Brzyski2019GroupPredictors} to give
 \begin{equation}\label{eqn:gslope_pen_mean}
	w_i^\text{mean} = \overline{F}^{-1}_{\chi_{p_j}} (1-q_gi/m), \; \text{where} \; \overline{F}_{\chi_{p_j}}(x):= \frac{1}{m}\sum_{j=1}^{m}F_{\chi_{p_j}}(\sqrt{p_j}x).
\end{equation}
The mean sequence weights defined in Equation \ref{eqn:gslope_pen_mean} are used for all gSLOPE numerical simulations in this manuscript (shown in Figure \ref{fig:appendix_gslope_weights}).
\begin{figure}[H]
\vskip 0.2in
\begin{center}
\centerline{\includegraphics[width=0.65\columnwidth]{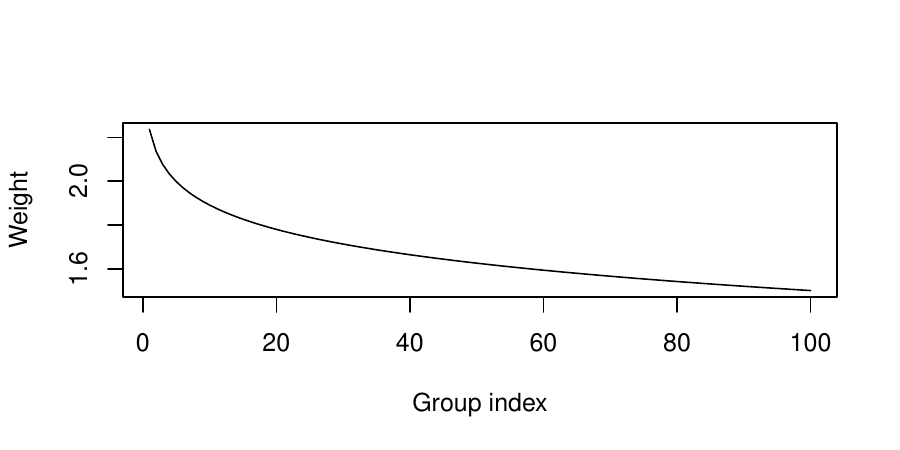}}
\caption{The gSLOPE weights, $w$, shown for Figure \ref{fig:time_fcn_of_p} for $p=500, m=100, q_g = 0.05$.}
\label{fig:appendix_gslope_weights}
\end{center}
\vskip -0.2in
\end{figure}

\subsection{SLOPE Subdifferential}\label{appendix:slope_subdiff}
The subdifferential for SLOPE was derived in \cite{Larsson2020a} and as it is a vital part of our arguments, it is reproduced here for ease of reference. 

Define the function $R:\mathbb{R}^p \rightarrow \mathbb{N}^p$ that returns the ranks of the absolute values of its input and the set $\mathcal{C}_i(\beta) = \{j\in \{1,\ldots,p\} \;:\; |\beta_i|=|\beta_j|\}$. Then, the subdifferential is given by \citep{Larsson2020a}

\begin{equation*}
    \partial J_\text{slope}(\beta; v) = \begin{cases}
       \left\{x\in \mathbb{R}^{|\mathcal{C}_i|}:\text{cumsum}(|x|_\downarrow -  v_{R(x)_{\mathcal{C}_i}}) \preceq 0 \right\}  &  \text{if} \; \beta_{\mathcal{C}_i} = 0, \\
\Big\{ x\in \mathbb{R}^{|\mathcal{C}_i|}: \text{cumsum}(|x|_\downarrow -  v_{R(x)_{\mathcal{C}_i}}) \preceq 0  \\ \;\;\;\;\; \; \text{and} \;\; \sum_{j\in\mathcal{C}_i}(|x_j| - v_{R(s)_j}) = 0 \\
 \;\;\;\;\; \;\text{and}\;\; \text{sign}(\beta_{\mathcal{C}_i}) = \text{sign}(x) \Big\}
 & \text{otherwise.}
    \end{cases}
\end{equation*}
The primary use of the subdifferential in this manuscript is the zero condition (Equation \ref{eqn:gslope_sub_condition}).
\subsection{Theory}\label{appendix:gslope_theory}
\begin{proof}[Proof of Theorem \ref{thm:gslope_subdiff}]
The proof is similar to that of Theorem 2.7 in \cite{Brzyski2019GroupPredictors}, where the subdifferential of gSLOPE is derived under equal groups. It is derived here under more general terms. The subdifferential needs to be derived under two cases: 
\begin{enumerate}
    \item Inactive groups, $\mathcal{G}_\mathcal{Z}$.
    \item Active groups, $\mathcal{G}_\mathcal{A}$.
\end{enumerate}
\textit{Case 1:} For inactive groups, we consider the subdifferential at zero. The subdifferential of a norm at zero is given by the dual norm of the unit ball \citep{Schneider2022TheEstimation},
\begin{equation*}
        \partial J_\text{gslope}(\mathbf{0};w)=\mathbf{B}_{J_\text{gslope}^*(\mathbf{0};w)}[0,1] = \{x:J_\text{gslope}^*(x;w) \leq 1\}.
\end{equation*}
The dual norm for gSLOPE is given by \citep{Brzyski2019GroupPredictors}
\begin{equation*}
    J_\text{gslope}^*(x;w) = J_\text{slope}^*([x]_{\mathcal{G},-0.5}).
\end{equation*}
Hence, the dual norm unit ball is
\begin{equation*}   
\mathbf{B}_{J_\text{gslope}^*(\mathbf{0};w)}[0,1] = \{x: [x]_{\mathcal{G},-0.5} \in  \mathbf{B}_{J_\text{slope}^*(\mathbf{0};w)}[0,1]\},
\end{equation*}
where $\mathbf{B}_{J_\text{slope}^*(\mathbf{0};w)}[0,1] = \{x\in \mathbb{R}^{m} : \text{cumsum}(|x|_\downarrow - w) \preceq \mathbf{0}\}$ is the unit ball of the dual norm to $J_{\text{slope}}$ \citep{Bogdan2015SLOPEAdaptiveOptimization}. Using this, the subdifferential at zero for the inactive groups, $\mathcal{Z}$, is given by
\begin{equation*}
    \partial J_{\text{gslope}}(\mathbf{0};w_{\mathcal{Z}}) =  \{x \in \mathbb{R}^{\card(\mathcal{G}_\mathcal{Z})}: [x]_{\mathcal{G}_\mathcal{Z},-0.5} \in \partial J_\text{slope}(\mathbf{0};w_{\mathcal{Z}})\}.
\end{equation*}

\textit{Case 2:} Without loss of generality, denote the group index $s$ such that $\|\beta^{(g)}\|_2 = 0$ for $g > s$ (inactive groups) and $\|\beta^{(g)}\|_2 \neq 0$ for $g \leq s$ (active groups). In other words,  $g\in\mathcal{G}_\mathcal{A}$ if $g\leq s$. Define a set $D=\{d \in \mathbb{R}^p: \|\beta^{(1)}+d^{(1)}\|_2 > \ldots > \|\beta^{(s)}+d^{(s)}\|_2, \|\beta^{(s)}+d^{(s)}\|_2 > \|d^{(g)}\|_2, g>s \}$. By definition of a subdifferential, if $x \in \partial J_{\text{gslope}}(\beta;w)$, then for all $d\in D$
\begin{equation*}
    \sum_{g=1}^m \sqrt{p_g} w_g \|\beta^{(g)} + d^{(g)}\|_2\geq \sum_{g=1}^m \sqrt{p_g}w_g \|\beta^{(g)}\|_2+ x^\top d.
\end{equation*}
Splitting this up into whether the groups are active (whether $g\leq s$):
\begin{align}\label{eqn:gslope_sub_eqn_1}
        \sum_{g=1}^s\sqrt{p_g} w_g \|\beta^{(g)} + d^{(g)}\|_2 +  \sum_{g=s+1}^m \sqrt{p_g} w_g\| d^{(g)}\|_2  \geq &\sum_{g=1}^s\sqrt{p_g}w_g  \|\beta^{(g)}\|_2\\ &+  \sum_{g=1}^s x^{(g)T} d^{(g)}+\sum_{g=s+1}^m x^{(g)T} d^{(g)}. \nonumber
\end{align}
Now, for $g\in\mathcal{G}_\mathcal{A}$, define a new set $D_g = \{d\in D: d^{(j)} \equiv \mathbf{0}, j\neq g\}$. Taking $d\in D_g$, Equation \ref{eqn:gslope_sub_eqn_1} becomes 
\begin{equation*}
    \sqrt{p_g}w_g \|\beta^{(g)} + d^{(g)}\|_2 \geq \sqrt{p_g}w_g\|\beta^{(g)}\|_2 + x^{(g)T}d^{(g)}.
\end{equation*}
Since the set $\{d^{(g)}: d \in D_g\}$ is open in $\mathbb{R}^{p_g}$ and contains zero, by Corollary G.1 in \cite{Brzyski2019GroupPredictors}, it follows that $x^{(g)}\in \partial f_g(b^{(g)})$ for $f_g: \mathbb{R}^{p_g} \rightarrow \mathbb{R}, f_g(x)=w_g \sqrt{p_g}\|x\|_2$. Now, for $g\leq s$, $f_g$ is differentiable in $\beta^{(g)}$, giving
\begin{equation*}
    x^{(g)} = w_g \sqrt{p_g}\frac{\beta^{(g)}}{\|\beta^{(g)}\|_2},
\end{equation*}
proving the result.
\end{proof}
\begin{proof}[Proof of Proposition \ref{propn:gslope_seq_strong}]
Suppose we have $\mathcal{B} \neq \emptyset$ after running the algorithm. Then, plugging in $h(\lambda_{k+1}) = ([\nabla f(\hat\beta(\lambda_{k+1}))]_{\mathcal{G},-0.5})_\downarrow$ gives
\begin{equation*}
 \text{cumsum}\Bigl(\bigl(([\nabla f(\hat\beta(\lambda_{k+1}))]_{\mathcal{G},-0.5})_\downarrow\bigr)_\mathcal{B} - \lambda_{k+1}w_\mathcal{B}\Bigr) \prec \mathbf{0},
    \end{equation*}
so that by the gSLOPE subdifferential (Theorem \ref{thm:gslope_subdiff}) all groups in $\mathcal{B}$ are inactive. This is valid by the KKT conditions (Equation \ref{eqn:kkt_condition}), as we know that $-\nabla f(\hat\beta(\lambda_{k+1})) \in \partial J_\text{gslope}(\mathbf{0};w)$. Hence, $\mathcal{S}_g(\lambda_{k+1})$ will contain the active set $\mathcal{A}_g(\lambda_{k+1})$.
\end{proof}
\begin{proof}[Proof of Proposition \ref{propn:gslope_seq_strong_grad_approx}]
Since $\text{cumsum}(y) \succeq \text{cumsum}(x) \iff y \succeq x$ \citep{Larsson2020a}, we only need to show for a group $g$,
\begin{equation*}
|h_g(\lambda_{k+1})| \leq |h_g(\lambda_{k})|+ \lambda_k w_g - \lambda_{k+1}w_g.
\end{equation*}
Applying the reverse triangle inequality to the Lipschitz assumption gives
\begin{align*}
 &|h_g(\lambda_{k+1})| - |h_g(\lambda_{k})| \leq \left|h_g(\lambda_{k+1}) -h_g(\lambda_{k}) \right|\leq \lambda_k w_g - \lambda_{k+1}w_g\\
\implies  &|h_g(\lambda_{k+1})| \leq  |h_g(\lambda_{k})| + \lambda_k w_g - \lambda_{k+1}w_g,
\end{align*}
proving the result.
\end{proof}
\subsection{KKT Checks}\label{appendix:gslope_kkt}
To check whether a group has been correctly discarded during the screening step, the KKT conditions for gSLOPE are checked. They are given by
\begin{align*}
    \mathbf{0} &\in \nabla f(\beta) + \lambda\partial J_\text{gslope}(\beta;w) \\
    \implies -\nabla f(\beta) &\in \lambda\partial J_\text{gslope}(\beta;w).
\end{align*}
Hence, we are checking whether the gradient of the loss function sits within the set of the gradient of the penalty. As we are only interested in identifying incorrectly discarded groups, we require only to check the subdifferential condition at zero. Hence, a violation occurs if a group is discarded but
\begin{align*}
&-\nabla f(\beta) \notin \lambda\partial J_\text{gslope}(\mathbf{0};w_{\mathcal{G}_\mathcal{Z}})\\
    \implies&-\nabla f(\beta) \notin \left\{x\in \mathbb{R}^{\card\mathcal{G}_\mathcal{Z}} : [x]_{\mathcal{G}_\mathcal{Z},-0.5} \in \partial J_\text{slope}(0;\lambda w_{\mathcal{G}_\mathcal{Z}})\right\} \\
    \implies &[\nabla f(\beta)]_{\mathcal{G}_\mathcal{Z},-0.5} \notin \partial J_\text{slope}(0; \lambda w_{\mathcal{G}_\mathcal{Z}}) \\
      \implies &\text{cumsum}(([\nabla f(\beta)]_{\mathcal{G}_\mathcal{Z},-0.5})_\downarrow -\lambda w_{\mathcal{G}_\mathcal{Z}}) \succ 0.
\end{align*}

\subsection{Path Start Proof}\label{appendix:gslope_path_derivation}
\begin{proof}[Proof of Proposition \ref{propn:gslope_path_start}]
The aim is to find the value of $\lambda$ at which the first group enters the model. When all features are zero, the gSLOPE KKT conditions (Equation \ref{eqn:kkt_condition}) are
\begin{equation*}
 \mathbf{0} \in \nabla f(\mathbf{0})+\lambda \partial J_\text{gslope}(\mathbf{0};w).
\end{equation*}
This is satisfied when
\begin{equation*}
    [\nabla f(\mathbf{0})]_{\mathcal{G},-0.5} \in \partial J_\text{slope}(\mathbf{0};\lambda w)  \implies \text{cumsum}\bigl(([\nabla f(\mathbf{0})]_{\mathcal{G},-0.5})_\downarrow - \lambda w\bigr) \preceq \mathbf{0}.
\end{equation*}
Rearranging this gives
\begin{equation*}
  \lambda\succeq \text{cumsum}\bigl(([\nabla f(\mathbf{0})]_{\mathcal{G},-0.5})_\downarrow\bigr) \oslash \text{cumsum}(w).
\end{equation*}
Picking the maximum possible $\lambda$ such that this holds yields
\begin{equation*}
    \lambda_1 = \max \left\{\text{cumsum}\bigl( ([\nabla f(\mathbf{0})]_{\mathcal{G},-0.5})_\downarrow\bigr)\oslash \text{cumsum}(w)\right\}.
\end{equation*}
This can be verified by noting that $\lambda_1 = J^*_\text{gslope}(\nabla f(\mathbf{0}); w)$ \citep{Ndiaye2016GapPenalties}. Now, $J^*_\text{gslope}(x;w) = J^*_\text{slope}([x]_{\mathcal{G},-0.5};w)$ \citep{Brzyski2019GroupPredictors}. The dual norm of SLOPE is given by \citep{Negrinho2014OrbitRegularization} 
\begin{equation*}
    J^*_\text{slope}(x;w) = \max \left\{\text{cumsum}(|x|_\downarrow) \oslash \text{cumsum}(w)\right\}.
\end{equation*}
Therefore, $\lambda_1$ is as before.
\end{proof}

\section{SPARSE-GROUP SLOPE}
\subsection{Penalty Weights}\label{appendix:sgs_pen_seq}
The penalty weights for SGS provide variable and group FDR-control simultaneously, under orthogonal designs \citep{Feser2023Sparse-groupFDR-control}. They are given by (where the indexing corresponds to the sorted variables/groups)
\begin{align*}
	&v_i^\text{max} = \max_{j=1,\dots,m} \left\{\frac{1}{\alpha} F_\mathcal{N}^{-1} \left(1-\frac{q_vi}{2p}\right) -   \frac{1}{3\alpha}(1-\alpha) a_j w_j\right\}, \; i=1,\dots,p, \\
	&w_i^\text{max} =\max_{j=1,\dots,m}\left\{\frac{F_\text{FN}^{-1}(1-\frac{q_gi}{m})-\alpha \sum_{k \in \mathcal{G}_j}v_k }{(1-\alpha) p_j}\right\}, \; i=1,\dots,m,
\end{align*}
where $F_{\chi_{p_j}}$ is the cumulative distribution function of a $\chi$ distribution with $p_j$ degrees of freedom, $F_\mathcal{N}$ is the cumulative distribution function of a folded Gaussian distribution, and $a_j$ is a quantity that requires estimation. The estimator $\hat{a}_j = \lfloor\alpha p_j\rfloor$ is proposed in \cite{Feser2023Sparse-groupFDR-control}. As with gSLOPE (Appendix \ref{appendix:gslope_pen_seq}), a relaxtion is possible, giving the weights
\begin{align}
	&v_i^\text{mean} = \overline{F}_\mathcal{N}^{-1}\left(1-\frac{q_vi}{2p}\right), \; \text{where}\; \overline{F}_\mathcal{N}(x) := \frac{1}{m}\sum_{j=1}^{m} F_\mathcal{N}\left(\alpha x +  \frac{1}{3}(1-\alpha) a_j w_j\right),\label{eqn:sgs_var_pen_mean} \\
	&w_i^\text{mean} = \overline{F}_\text{FN}^{-1}\left(1-\frac{q_gi}{p}\right), \; \text{where}\; \overline{F}_\text{FN}(x) := \frac{1}{m}\sum_{j=1}^{m} F_\text{FN}\left((1-\alpha) p_j x + \alpha \sum_{k \in \mathcal{G}_j} v_k\right).\label{eqn:sgs_grp_pen_mean}
\end{align}
In the manuscript, as recommended by \cite{Feser2023Sparse-groupFDR-control} under general settings, the SGS variable mean (Equation \ref{eqn:sgs_var_pen_mean}) and gSLOPE group mean (Equation \ref{eqn:gslope_pen_mean}) weights are used for all SGS numerical simulations.
\begin{figure}[H]
\vskip 0.2in
\begin{center}
\centerline{\includegraphics[width=0.8\columnwidth]{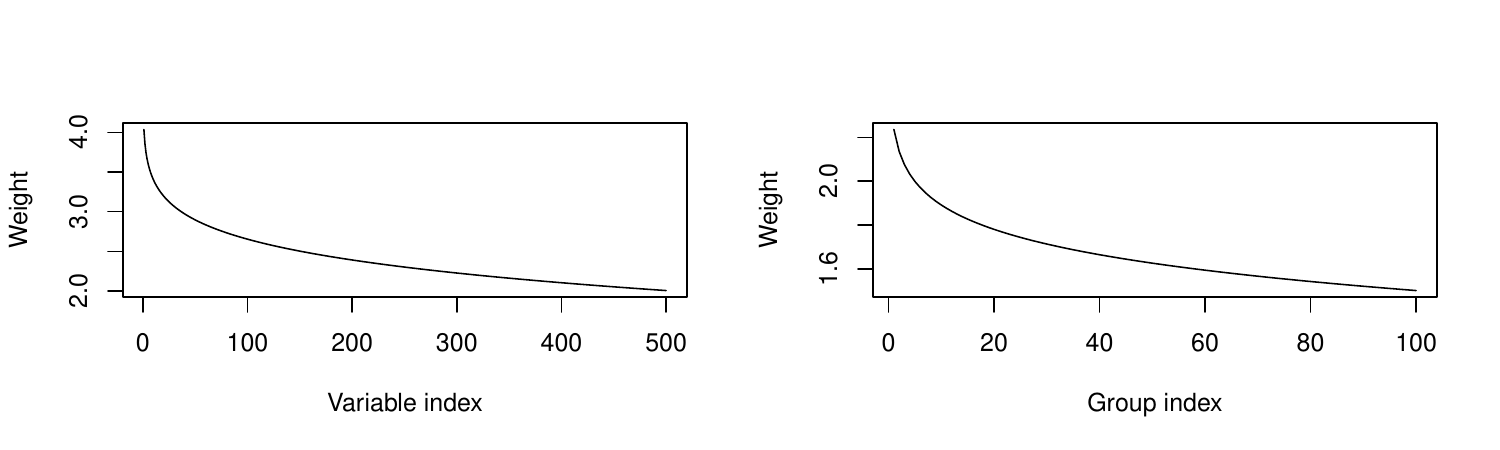}}
\caption{The SGS weights, $(v,w)$, shown for Figure \ref{fig:time_fcn_of_p} for $p=500, m=100, q_v=0.05, q_g = 0.05,\alpha=0.95$.}
\label{fig:appendix_sgs_weights}
\end{center}
\vskip -0.2in
\end{figure}
\subsection{Derivation of Soft Thresholding Operator}\label{appendix:sgs_sto}

\begin{proof}[Proof of Lemma \ref{lemma:sto}]
To determine the form of the quantity $\partial J_\text{slope}(\mathbf{0};v)$, consider that for Equation \ref{eqn:group_condition_sgs} to be satisfied, the term inside the $[\cdot]$ operator needs to be as small as possible. Now, 
\begin{equation*}
    \partial J_\text{slope}(\mathbf{0};v) = \{y: \text{cumsum}(|y|) \preceq \text{cumsum}(v)\}.
\end{equation*}
Note that $\text{cumsum}(y) \preceq \text{cumsum}(x) \iff y  \preceq x$. We consider the cases:
\begin{enumerate}
    \item $\nabla_i f(\beta) > \lambda\alpha v_i$: choose $y_i = -v_i$.
    \item $\nabla_i f(\beta) < -\lambda\alpha v_i$: choose $y_i = v_i$.
    \item $\nabla_i f(\beta) \in [-\lambda\alpha v_i,\lambda\alpha v_i]$: choose $y_i = \nabla_i f(\beta)/\lambda\alpha v_i$.
\end{enumerate}
Hence, the term becomes
\begin{equation*}
    S(\nabla f(\beta),\lambda\alpha v) := \text{sign}(\nabla f(\beta))(|\nabla f(\beta)| - \lambda\alpha v)_+,
\end{equation*}
which is the soft thresholding operator. 
\end{proof}
\subsection{Theory}\label{appendix:sgs_theory}
\begin{proposition}[Strong group screening rule for SGS]\label{propn:sgs_screen}
Let $\tilde{h}(\lambda) := ([S(\nabla f(\beta),\lambda\alpha v)]_{\mathcal{G},-0.5})_\downarrow$. Then taking $c=\tilde{h}(\lambda_{k+1})$ and $\phi = (1-\alpha)\lambda_{k+1} w$ as inputs for Algorithm \ref{alg:slope_screen_alg} returns a superset $\mathcal{S}_g(\lambda_{k+1})$ of the active set $\mathcal{A}_g(\lambda_{k+1})$.
\end{proposition}
\begin{proof}[Proof of Proposition \ref{propn:sgs_screen}]
The proof is similar to that of Proposition \ref{propn:gslope_seq_strong}. Suppose we have $\mathcal{B} \neq \emptyset$ after running the algorithm. Then,
    \begin{align*}
   \text{cumsum}(\tilde{h}_\mathcal{B}(\lambda_{k+1}) - \lambda_{k+1}(1-\alpha) w_\mathcal{B}) &\prec \mathbf{0} \\ \implies \text{cumsum}\Bigl(\bigl([ S(\nabla f(\beta),\lambda_{k+1}\alpha v)]_{\mathcal{G},-0.5})_\downarrow\bigr)_\mathcal{B} - \lambda_{k+1}(1-\alpha) w_\mathcal{B}\Bigr) &\prec \mathbf{0},
    \end{align*}
so that by the SGS subdifferential (Equation \ref{eqn:group_condition_sgs}) all groups in $\mathcal{B}$ are inactive. Hence, $\mathcal{S}_g(\lambda_{k+1})$ will contain the active set $\mathcal{A}_g(\lambda_{k+1})$.
\end{proof}
\begin{proof}[Proof of Proposition \ref{propn:sgs_screen_grad_approx}]
The proof is identical to that of Proposition \ref{propn:gslope_seq_strong_grad_approx}, replacing $h_g(\cdot)$ with $\tilde{h}_g(\cdot)$ and $\lambda_{k+1}w$ by $\lambda_{k+1}(1-\alpha)w$.
\end{proof}
\begin{proposition}[Strong variable screening rule for SGS]\label{propn:sgs_screen_var}
    Let $\bar{h}(\lambda) = |(\nabla f(\hat\beta(\lambda)))|_\downarrow$. Then taking $c = \bar{h}(\lambda_{k+1})$ and $\phi =  \lambda_{k+1}\alpha v$ for only the variables contained in the groups in $\mathcal{A}_g(\lambda_{k+1})$ in Algorithm \ref{alg:slope_screen_alg} returns a superset $\mathcal{S}_v(\lambda_{k+1})$ of the active set $\mathcal{A}_v(\lambda_{k+1})$.
\end{proposition}
\begin{proof}
      Suppose we have $\mathcal{B} \neq \emptyset$ after running the algorithm. Then, we have
 \begin{equation*}
\text{cumsum}( \bar{h}_\mathcal{B}(\lambda_{k+1}) -  \lambda_{k+1}\alpha v_\mathcal{B})\prec \mathbf{0} \implies 
\text{cumsum}\Bigl( \bigl(|\nabla f(\hat\beta(\lambda_{k+1}))|_\downarrow\bigr)_\mathcal{B} -  \lambda_{k+1}\alpha v_\mathcal{B}\Bigr)\prec \mathbf{0},
 \end{equation*}
so that by the SGS subdifferential for non-zero groups (Equation \ref{eqn:sgs_non_zero_grp}) all variables in $\mathcal{B}$ are inactive. Hence, $\mathcal{S}_v(\lambda_{k+1})$ will contain the active set $\mathcal{A}_v(\lambda_{k+1})$.
\end{proof}
\begin{proof}[Proof for Proposition \ref{propn:sgs_screen_var_grad_approx}]
The proof is identical to that of Proposition \ref{propn:gslope_seq_strong_grad_approx}, replacing $h_g(\cdot)$ with $\bar{h}_g(\cdot)$, $\lambda_{k+1} v$ with $\lambda_{k+1} \alpha v$, and considering only variables in the groups contained in $\mathcal{A}_g(\lambda_{k+1})$.
\end{proof}
\subsection{KKT Checks}\label{appendix:sgs_kkt}
For SGS, the KKT conditions are first checked at the group-level for inactive groups (Appendix \ref{appendix:grp_checks}). Further variable checks are performed for violating groups and variables in active groups (to check whether the variables should also be active) (Appendix \ref{appendix:var_checks}). The violating variables from these secondary variable checks are added back into $\mathcal{E}_v$.

\subsubsection{Group Checks}\label{appendix:grp_checks}
A group violation occurs if the KKT conditions do not hold at the group-level (Equation \ref{eqn:group_condition_sgs}). That is, a violation occurs if a group is discarded but
\begin{equation*}
\text{cumsum}\Bigl( \bigl([\nabla f(\beta) +\lambda\alpha \partial J_\text{slope}(\mathbf{0};v)]_{\mathcal{G}_\mathcal{Z},-0.5}\bigr)_\downarrow -\lambda(1-\alpha)w_\mathcal{Z}\Bigr) \succ \mathbf{0}.
\end{equation*}
\subsubsection{Variable Checks}\label{appendix:var_checks}
For the set of variables in a violating group (from Appendix \ref{appendix:grp_checks}), denoted $\mathcal{G}_{\mathcal{K}_g}$, a variable violation occurs if Equation \ref{eqn:sgs_non_zero_grp} does not hold. That is, if 
\begin{equation*}
 \nabla_{\mathcal{G}_{\mathcal{K}_g}} f(\beta) \notin \lambda \alpha\partial J_\text{slope}(\mathbf{0};v_{\mathcal{G}_{\mathcal{K}_g}})  \implies \text{cumsum}(|\nabla_{\mathcal{G}_{\mathcal{K}_g}} f(\beta)| -\lambda \alpha v_{\mathcal{G}_{\mathcal{K}_g}}) \succ 0.     
\end{equation*}
\subsubsection{Alternative KKT Checks} \label{appendix:alternative_checks}
An alternative approach for SGS is to check the KKT conditions directly on the variables. The KKT conditions (Equation \ref{eqn:sgs_kkt}) can be rewritten as
\begin{equation*}
     -\nabla f(\beta) -\lambda (1-\alpha) \partial J_\text{gslope}(\beta; w)  \in \lambda \alpha \partial J_\text{slope}(\beta; v). 
\end{equation*}
A KKT violation occurs the zero subdifferential conditions are not satisfied
\begin{align*}
        &-\nabla f(\beta) -\lambda (1-\alpha) \partial J_\text{gslope}(\beta; w)  \notin \lambda \alpha \partial J_\text{slope}(\mathbf{0}; v) \\
\implies&\text{cumsum}\left(\left|\nabla f(\beta) +\lambda (1-\alpha) \partial J_\text{gslope}(\beta; w)\right|_\downarrow - \lambda \alpha v\right) \succ \mathbf{0}.
\end{align*}
Now, the objective is to make the term inside the sorted absolute value operator as small as possible, given that the subdifferential term is unknown. To do this, a similar derivation as in Section \ref{appendix:sgs_sto} can be used to determine that the term must be the soft thresholding operator, so that a violation occurs if
\begin{equation*}
    \text{cumsum}\left(\left|S(\nabla f(\beta), \lambda(1-\alpha)\tau\omega) \right|_\downarrow - \lambda \alpha v\right) \succ \mathbf{0},
\end{equation*}
where $\tau$ and $\omega$ are expanded vectors of the group sizes ($\sqrt{p_g})$ and penalty weights ($w_g$) to $p$ dimensions, so that each variable within the same group is assigned the same value. However, as we have had to approximate the unknown subdifferential term, this check is not exact. In practice, we found that this check was not stringent enough (due to the approximation), leading to violations being missed.
\subsection{Path Start Proof}\label{appendix:sgs_path_derivation}
\begin{proof}[Proof of Proposition \ref{propn:sgs_path_start}]
The aim is to find the value of $\lambda$ at which the first variable enters the model. When all features are zero, the SGS KKT conditions (Equation \ref{eqn:sgs_kkt}) are
\begin{align*}
    &-\nabla f(\mathbf{0}) \in \lambda(1-\alpha)\partial J_\text{gslope}(\mathbf{0};w) + \lambda\alpha \partial J_\text{slope}(\mathbf{0};v) \\
     \implies&-\frac{1}{\lambda} \nabla f(\mathbf{0}) -(1-\alpha)\partial J_\text{gslope}(\mathbf{0};w) \in  \alpha \partial J_\text{slope}(\mathbf{0};v) \\
     \implies&\text{cumsum}\left( \left|-\frac{1}{\lambda} \nabla f(\mathbf{0}) -(1-\alpha)\partial J_\text{gslope}(\mathbf{0};w)\right|_\downarrow - \alpha v\right) \preceq  \mathbf{0}. 
\end{align*}
By the reverse triangle inequality and ordering of the group weights
\begin{align*}
      &\frac{1}{\lambda} \text{cumsum}\left(| \nabla f(\mathbf{0})|_\downarrow\right) \preceq  \text{cumsum}((1-\alpha) |\partial J_\text{gslope}(\mathbf{0};w)|- \alpha v) \\
     \implies&\lambda \succeq \text{cumsum}(| \nabla f(\mathbf{0})|_\downarrow) \oslash \text{cumsum}((1-\alpha) |\partial J_\text{gslope}(\mathbf{0};w)|- \alpha v).
\end{align*}
Now, note that for $x \in J_\text{gslope}(\mathbf{0};w)$, it holds
\begin{equation*}
     \text{cumsum}([x]_{\mathcal{G},-0.5} - w) \preceq \mathbf{0} \implies [x]_{\mathcal{G},-0.5} \preceq w \implies \|x^{(g)}\|_2 \leq \sqrt{p_g}w_g, \forall g \in \mathcal{G}.
\end{equation*}
This is satisfied at the upper limit at $x=\tau \omega$. Hence,
\begin{equation*}
    \lambda_1 = \max\left\{\text{cumsum}(| \nabla f(\mathbf{0})|_\downarrow)\oslash
    \text{cumsum}((1-\alpha)\tau\omega- \alpha v)\right\}.
\end{equation*}
\end{proof}
\section{SLOPE ALGORITHM} \label{appendix:slopealgorithm}
\begin{algorithm}[H]
   \caption{SLOPE subdifferential algorithm from \cite{Larsson2020a}}
\label{alg:slope_screen_alg}
\begin{algorithmic}
   \STATE {\bfseries Input:} $c \in \mathbb{R}^p, \phi \in \mathbb{R}^p$, where  $\phi_1 \geq \cdots \geq$
$\phi_p \geq 0$
\STATE $\mathcal{S}, \mathcal{B} \leftarrow \varnothing$
   \FOR{$i=1$ {\bfseries to} $p$}
   \STATE $\mathcal{B} \leftarrow \mathcal{B} \cup\{i\}$
   \IF{$\text{cumsum}(c_\mathcal{B} - \phi_\mathcal{B}) \geq 0$}
   \STATE $\mathcal{S} \leftarrow \mathcal{S} \cup \mathcal{B}$
   \STATE $\mathcal{B} \leftarrow \varnothing$
   \ENDIF
   \ENDFOR
\STATE {\bfseries Output: $\mathcal{S}$}
\end{algorithmic}
\end{algorithm}

\section{SCREEENING RULE FRAMEWORK}\label{appendix:framework}
\subsection{Group SLOPE Algorithm}\label{appendix:gslope_framework}
For the following is performed for $k = 1,\ldots,l-1$:
\begin{enumerate}
    \item  Set $\mathcal{E}_g = \mathcal{S}_g(\lambda_{k+1}) \cup \mathcal{A}_g(\lambda_{k})$, where $\mathcal{S}_g(\lambda_{k+1})$ is obtained using Proposition \ref{propn:gslope_seq_strong_grad_approx}.
    \item Compute $\hat\beta(\lambda_{k+1})$ by Equation \ref{eqn:general_problem} with the gSLOPE norm using only the groups in $\mathcal{E}_g.$ For any groups not in $\mathcal{E}_g$, $\hat\beta(\lambda_{k+1})$ is set to zero.
    \item Check the KKT conditions (Equation \ref{eqn:kkt_condition}) for all groups at this solution.
    \item If there are no violations, we are done and keep $\hat\beta(\lambda_{k+1})$. Otherwise, add the violating groups into $\mathcal{E}$ and return to Step 2.
\end{enumerate}
\subsection{SGS Algorithm}\label{appendix:sgs_framework}
For the following is performed for $k = 1,\ldots,l-1$:
\begin{enumerate}
    \item \textit{Group screen step}: Calculate $\mathcal{S}_g(\lambda_{k+1})$ using Proposition \ref{propn:sgs_screen_grad_approx}.
    \item \textit{Variable screen step}: Set $\mathcal{E}_v = \mathcal{S}_v(\lambda_{k+1}) \cup \mathcal{A}_v(\lambda_k)$, where $\mathcal{S}_v(\lambda_{k+1})$ is obtained using Proposition \ref{propn:sgs_screen_var_grad_approx} with only the groups in $\mathcal{S}_g(\lambda_{k+1})$.
    \item Compute $\hat\beta(\lambda_{k+1})$ by Equation \ref{eqn:general_problem} with the SGS norm using only the features in $\mathcal{E}_v$. For features not in $\mathcal{E}_v$, $\hat\beta(\lambda_{k+1})$ is set to zero. 
    \item Check the KKT conditions (Equation \ref{eqn:sgs_kkt}) for all features at this solution. 
    \item If there are no violations, we are done and keep $\hat\beta(\lambda_{k+1})$, otherwise add in the violating variables into $\mathcal{E}_g$ and return to Step 3.
\end{enumerate}

\newpage

\section{GROUP-BASED OSCAR}\label{appendix:oscar}
This section provides supplementary materials for extending the proposed screening rules to group-based OSCAR models.

\subsection{Penalty Sequence}\label{appendix:oscar_weights}
The gOSCAR and SGO weights are defined by (for a variable $i\in[p]$ and group $g\in[m]$) (Figure \ref{fig:appendix_oscar_weights})
\begin{equation}
    v_i = \sigma_1 + \sigma_2(p-i), \; w_g = \sigma_1 + \sigma_3(m-g), \; \sigma_3 = \sigma_1/m.
\end{equation}

\begin{figure}[!h]
\vskip 0.2in
\begin{center}
\centerline{\includegraphics[width=0.8\columnwidth]{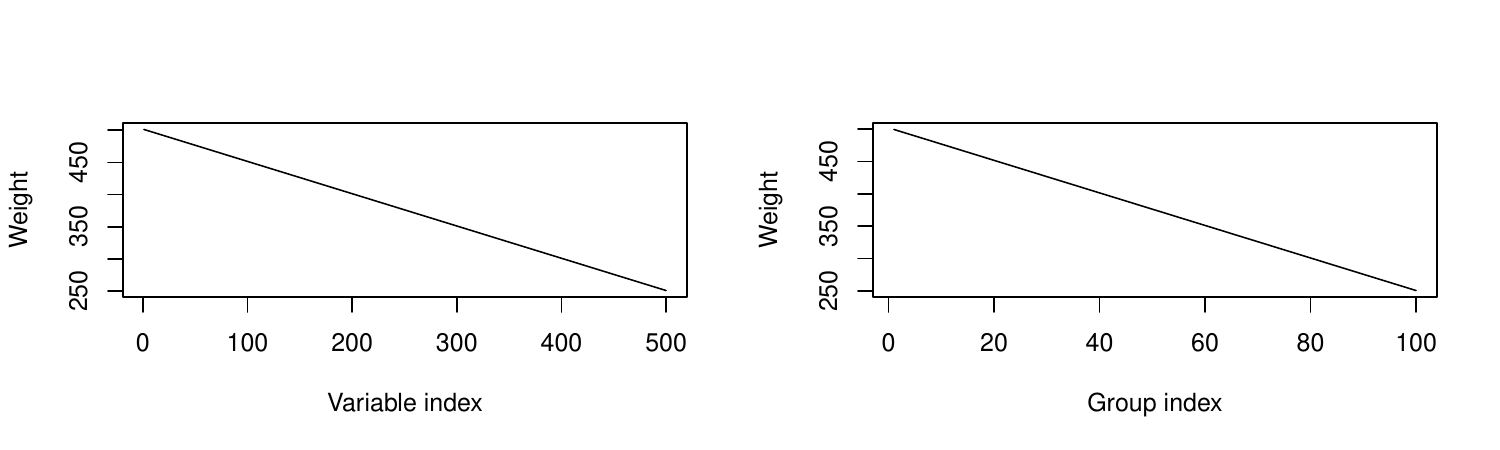}}
\caption{The SGO weights, $(v,w)$, for $p=500, m=100, q_v=0.05, q_g = 0.05,\alpha=0.95$.}
\label{fig:appendix_oscar_weights}
\end{center}
\vskip -0.2in
\end{figure}

\subsection{Results}
Observations and conclusions made for the screening rules of gSLOPE and SGS are also found to be true for gOSCAR and SGO (Figures \ref{fig:fig_1_oscar} - \ref{fig:fig_7_oscar}). 

Figure \ref{fig:fig_1_oscar} illustrates the effectiveness of bi-selection of SGO, similar to the effectiveness observed for SGS. Figures \ref{fig:fig_2_oscar} and \ref{fig:fig_3_oscar} showcase the efficiency of the screening rules on the proportion of the selected groups/variables. The screening rules are found to be effective across different data characteristics, with the running time of the models significantly decreasing (Figure \ref{fig:fig_4_oscar}). KKT violations for SGO are more common compared to gOSCAR (Figure \ref{fig:fig_7_oscar}), due to the additional assumptions made at the second screening layer (as with SGS). Similar to Figure \ref{fig:kkt_spikeplot_synth}, the shape of the increasing number of KKT violations mirrors the log-linear shape of the regularization path. 
\vspace{5mm}
\begin{figure}[!h]
\centering
\begin{minipage}[t]{.48\textwidth}
  \centering
  \includegraphics[width=\linewidth,valign=t]{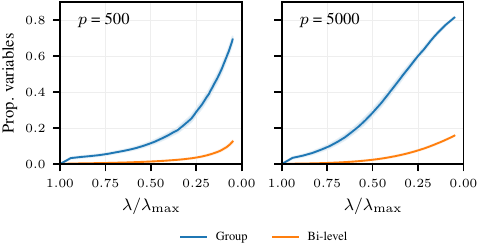}
  \captionof{figure}{The proportion of variables in $\mathcal{S}_v$ relative to the full input for SGO, shown for group and bi-level screening plotted as function of the regularization path, applied to the synthetic data (Section \ref{section:results_sim}). The data are generated under a linear model for $p = 500, 5000$. The results are averaged over 100 repetitions and 95\% confidence intervals are shown (the SGO equivalent of Figure \ref{fig:bi-level-screening}).}
    \label{fig:fig_1_oscar}
\end{minipage}%
\hfill
\begin{minipage}[t]{.48\textwidth}
 \centering
  \includegraphics[width=\linewidth,valign=t]{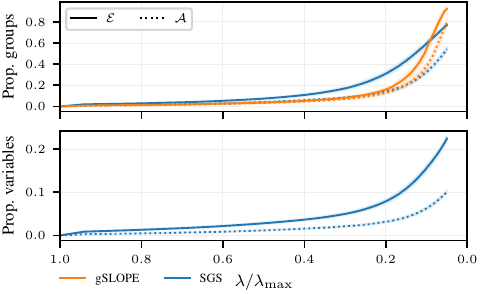}
  \captionof{figure}{The number of groups/variables in $\mathcal{E}, \mathcal{A}$ for both gOSCAR and SGO as a function of the regularization path for the linear model with $p=2750, \rho=0.6, m = 197$. The results are averaged over $100$ repetitions, with the shaded regions corresponding to $95\%$ confidence intervals (the gOSCAR/SGO equivalent of Figure \ref{fig:path_plot_gslope_sgs}).}
  \label{fig:fig_2_oscar}
\end{minipage}
\end{figure}

\begin{figure}[!h]
\centering
\begin{minipage}[t]{.48\textwidth}
  \centering
  \includegraphics[width=\linewidth,valign=t]{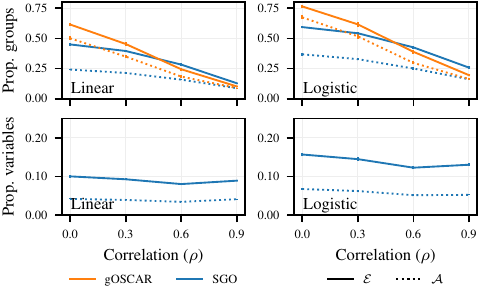}
  \captionof{figure}{The proportion of groups/variables in $\mathcal{E}, \mathcal{A}$, relative to the full input, shown for gOSCAR and SGO. This is shown as a function of the correlation ($\rho$), averaged over all cases of the input dimension ($p$), with $100$ repetitions for each $p$, for both linear and logistic models, with standard errors shown (the gOSCAR/SGO equivalent of Figure \ref{fig:gaussian_vs_log}).}
    \label{fig:fig_3_oscar}
\end{minipage}%
\hfill
\begin{minipage}[t]{.48\textwidth}
 \centering
  \includegraphics[width=\linewidth,valign=t]{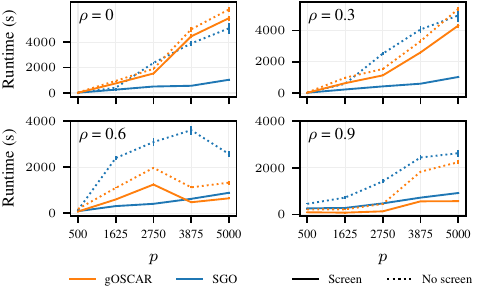}
  \captionof{figure}{Runtime (in seconds) for fitting $50$ models along a path, shown for screening against no screening as a function of $p$, broken down into different correlation cases, for the linear model. The results are averaged over $100$ repetitions, with standard errors shown (the gOSCAR/SGO equivalent of Figure \ref{fig:time_fcn_of_p}).}
  \label{fig:fig_4_oscar}
\end{minipage}
\end{figure}

\begin{figure}[!h]
\begin{center}
\centerline{\includegraphics[width=.6\columnwidth]{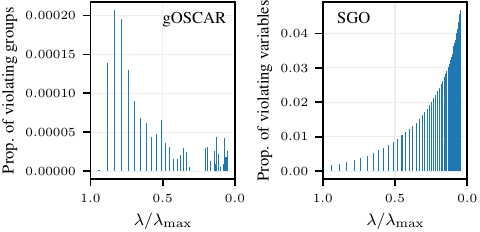}}
\caption{The proportion of KKT violations relative to the full input space, as a function of the regularization path. Group violations for gOSCAR and variable violations for SGS, under linear models, averaged over all cases of $p$ and $\rho$ (the gOSCAR/SGO equivalent of Figure \ref{fig:kkt_spikeplot_synth}).}
\label{fig:fig_7_oscar}
\end{center}
\end{figure}

\clearpage
\newpage 

\section{RESULTS}\label{appendix:results}
\subsection{Computational Information}\label{appendix:comp_info}
The simulated experiments were executed on a high-performance computing cluster (x86-64 Linux GNU) and the real data analysis was conducted on a Apple Macbook Air (M1, 8GB). Code for all simulations is available in the Supplementary Material. For all models, ATOS was used with the algorithmic parameters given in Table \ref{tbl:atos_params}.

\begin{table}[H]
\centering
  \caption{Hyperparameters used for running ATOS in the synthetic and real data studies.}
\label{tbl:atos_params}
  \begin{tabular}{lll}
    \toprule
    Parameter     & Synthetic Data     & Real Data \\
    \midrule
Max iterations & 5000 & 10000 \\
Backtracking & 0.7 & 0.7 \\
Max backtracking iterations & 100 & 100 \\
Convergence tolerance & $10^{-5}$ & $10^{-5}$ \\
Convergence criteria & $\|x - z\|_2$ & $\|x - z\|_2$ \\
Standardization & $\ell_2$ & $\ell_2$ \\
Intercept & Yes for linear & Yes for linear \\
Warm starts & Yes & Yes \\
    \bottomrule
  \end{tabular}
  \end{table}

\subsection{Solution Optimality}
\label{appendix:comp_real_info}

This section presents the accuracy of the models with and without screening, by comparing the $\ell_2$ distances observed between the screened and non-screened fitted values. 

\paragraph{Synthetic Data} For the linear model, the maximum $\ell_2$ distances observed between the screened and non-screened fitted values were of order $10^{-6}$ for gSLOPE and $10^{-9}$ for SGS (Table \ref{tbl:full_table_grp_linear}). Across the different cases, 98000 models were fit in total for each approach (excluding the models for $\lambda_1$, where no screening is applied). Of these model fits, there were no instances for gSLOPE where $\mathcal{E}$ was not a superset of $\mathcal{A}$. There was only one instance (out of the 98000) that this occurred for SGS, where $\mathcal{E}$ was missing a single variable contained in $\mathcal{A}$ (which had a non-screen fitted value of $\hat\beta = -0.004$).  

For the logistic model, the maximum $\ell_2$ distances observed between the screened and non-screened fitted values were of order $10^{-8}$ for gSLOPE and $10^{-9}$ for SGS (Table \ref{tbl:full_table_grp_log}). Across the different cases, 98000 models were fit in total for each approach (excluding the models for $\lambda_1$, where no screening is applied). Of these model fits, there were no instances for gSLOPE or SGS where $\mathcal{E}$ was not a superset of $\mathcal{A}$.

\paragraph{Real Data} In the real data analysis, the estimated coefficients with and without screening were very close to each other for both SGS and gSLOPE (Table \ref{tbl:full_table_grp_real_data}). However, direct comparison is less meaningful here, as the models often failed to converge without screening, therefore not reaching the optimal solution.

\subsection{Additional Results from the Simulation Study}\label{appendix:results_sim_study}

\begin{sidewaystable}
\begin{table}[H]
\caption{Variable screening metrics for SGS using linear and logistic models for the simulation study presented in Section \ref{section:results_sim}. The number of variables in $\mathcal{A}_v, \mathcal{S}_v, \mathcal{E}_v$, and $\mathcal{K}_v$ are shown, averaged across all $20$ cases of the correlation ($\rho$) and $p$. Standard errors are shown.}
\begin{center}
\begin{small}
\begin{sc}
\begin{tabular}{llrrrr}
\toprule
 Method & Type & $\card(\mathcal{A}_v)$ & $\card(\mathcal{S}_v)$ &  $\card(\mathcal{E}_v)$& $\card(\mathcal{K}_v)$\\
\midrule
  SGS & Linear& $179\pm3$& $313\pm5$  & $363\pm6$ & $51\pm1$\\
 SGS & Logistic & $230\pm3$& $405\pm5$ & $472\pm6$ & $66\pm1$\\
\bottomrule
\end{tabular}
\label{tbl:full_table_var_linear}
\end{sc}
\end{small}
\end{center}
\end{table}
\vspace{20pt}

\begin{table}[H]
\caption{General and group screening metrics for SGS and gSLOPE using linear and logistic models for the simulation study presented in Section \ref{section:results_sim}. General metrics: the runtime (in seconds) for screening against no screening, the number of fitting iterations for screening against no screening, and the $\ell_2$ distance between the fitted values obtained with screening and no screening. Group screening metrics: the number of groups in $\mathcal{A}_g, \mathcal{S}_g, \mathcal{E}_g$, and $\mathcal{K}_g$. The results are averaged across all $20$ cases of the correlation ($\rho$) and $p$. Standard errors are shown.}
\begin{center}
\begin{small}
\begin{sc}
\begin{tabular}{llrrrrrrrrr}
\toprule
Method &Type&\multicolumn{1}{p{1.5cm}}{\raggedleft Runtime \\ screen (s)} & \multicolumn{1}{p{1.5cm}}{\raggedleft Runtime \\ no screen (s)} &$\card(\mathcal{A}_g)$ & $\card(\mathcal{S}_g)$ &  $\card(\mathcal{E}_g)$& $\card(\mathcal{K}_g)$ &\multicolumn{1}{p{1.5cm}}{\raggedleft Num it \\ screen} & \multicolumn{1}{p{1.5cm}}{\raggedleft Num it \\ no screen } & \multicolumn{1}{p{1.5cm}}{\raggedleft $\ell_2$ dist \\ to no screen} \\\\
\midrule
gSLOPE & Linear&$ 1016\pm21$&$ 1623\pm27$&$ 55\pm1  $&$ 76\pm1 $&$ 76\pm1 $&$ 0.006\pm0.004 $&$ 333\pm6$&$351\pm6 $&$2\times10^{-6}\pm1\times 10^{-6}$\\
gSLOPE & Logistic&$ 814\pm8$&$ 1409\pm11$&$ 71\pm1 $&$ 97\pm1 $&$ 97\pm1 $&$ 0.014\pm0.014 $&$ 78\pm1$&$83\pm1 $&$1\times10^{-8}\pm1 \times 10^{-8}$\\
SGS & Linear&$ 735\pm15$&$ 1830\pm34$&$ 61\pm1 $&$ 84\pm1 $&$ 91\pm1 $&-&$ 91\pm3 $&$708\pm12$&$2\times10^{-9}\pm3\times 10^{-9}$\\
SGS & Logistic&$ 407\pm2$&$ 859\pm6$&$ 84\pm1 $&$ 107\pm1 $&$ 118\pm1 $&-&$ 7\pm0.2 $&$51\pm0.8$&$4\times10^{-9}\pm3 \times 10^{-10}$\\
\bottomrule
\end{tabular}
\label{tbl:full_table_grp_linear}
\end{sc}
\end{small}
\end{center}
\vskip -0.1in
\end{table}
\end{sidewaystable}

\newpage

\subsubsection{Additional Results for the Linear Model}\label{appendix:results_linear}
\vspace{5mm}
\begin{figure}[H]
\centering
  \includegraphics[width=.6\linewidth]{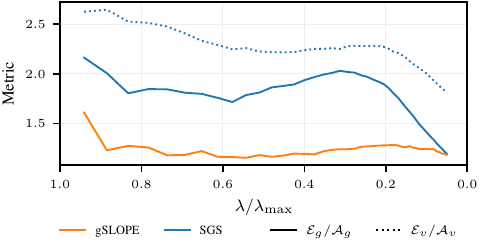}
  \caption{The proportion of groups/variables in $\mathcal{E}, \mathcal{A}$, relative to the full input, for gSLOPE and SGS, as a function of the regularization path for the linear model with $p=2750, \rho=0.6, m = 197$. The results are averaged over $100$ repetitions.}
   \label{fig:path_plot_gslope_sgs_metric_2}
\end{figure}
\vspace{5mm}
\begin{figure}[H]
  \begin{subfigure}[t]{.49\textwidth}
    \centering
    \includegraphics[width=\linewidth]{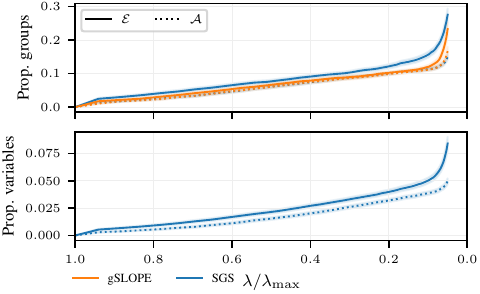}
    \caption{$p=500, \rho = 0$}
  \end{subfigure}
  \hfill
  \begin{subfigure}[t]{.49\textwidth}
    \centering
    \includegraphics[width=\linewidth]{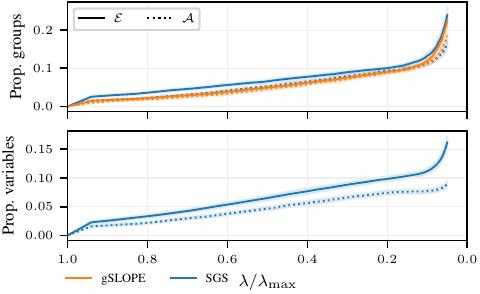}
    \caption{$p=500, \rho = 0.9$}
\vspace*{2mm}
  \end{subfigure}

  \begin{subfigure}[t]{.49\textwidth}
    \centering
    \includegraphics[width=\linewidth]{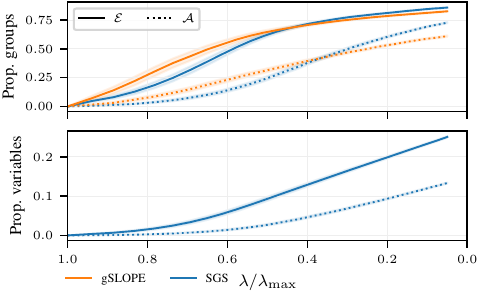}
    \caption{$p=5000, \rho = 0$}
  \end{subfigure}
  \hfill
  \begin{subfigure}[t]{.49\textwidth}
    \centering
    \includegraphics[width=\linewidth]{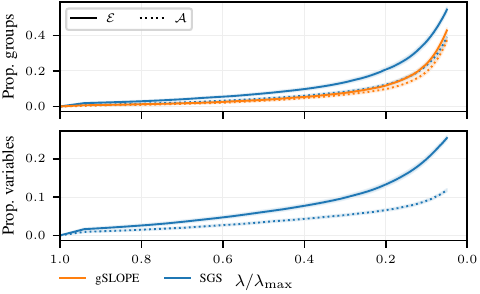}
    \caption{$p=5000, \rho = 0.9$}
  \end{subfigure}
  \caption{The number of groups/variables in $\mathcal{E}, \mathcal{A}$ as a function of the regularization path for the linear model with SGS and gSLOPE, shown for different values of the correlation ($\rho$) and $p$. The results are averaged over $100$ repetitions, with $95\%$ confidence intervals shown. 
  }
  \label{fig:appendix_linear}
\end{figure}

\newpage

\begin{sidewaystable}
\begin{table}[H]
\caption{Variable screening metrics for SGS using a linear model for the simulation study presented in Section \ref{section:results_sim}. The number of variables in $\mathcal{A}_v, \mathcal{S}_v, \mathcal{E}_v$, and $\mathcal{K}_v$ are shown. The results are shown for different values of $p$, averaged across $\rho \in \{0,0.3,0.6,0.9\}$. Standard errors are shown.}
\vskip 0.15in
\begin{center}
\begin{small}
\begin{sc}
\begin{tabular}{lrrrrr}
\toprule
 Method & $p$ &$\card(\mathcal{A}_v)$ & $\card(\mathcal{S}_v)$ &  $\card(\mathcal{E}_v)$& $\card(\mathcal{K}_v)$\\
\midrule
SGS&$ 500 $&$ 19\pm1 $&$ 24\pm1 $&$ 28\pm1 $&$ 4\pm0.2$  \\
  SGS &$ 1625 $&$ 83\pm3 $&$ 138\pm5 $&$ 161\pm6 $&$ 22\pm1$ \\  
  SGS &$ 2750 $&$ 188\pm7 $&$ 316\pm10 $&$ 370\pm12 $&$ 54\pm2 $\\ 
  SGS &$ 3875 $&$ 268\pm9 $&$ 470\pm14 $&$ 548\pm16 $&$ 78\pm3$  \\
  SGS &$ 5000 $&$ 334\pm10 $&$ 618\pm17 $&$ 712\pm20 $&$ 95\pm3 $ \\
\bottomrule
\end{tabular}
\label{tbl:full_table_var_linear_p}
\end{sc}
\end{small}
\end{center}
\vskip -0.1in
\end{table}

\vspace{20pt}

\begin{table}[H]
\caption{General and group screening metrics for SGS and gSLOPE using linear models for the simulation study presented in Section \ref{section:results_sim}. General metrics: the runtime (in seconds) for screening against no screening, the number of fitting iterations for screening against no screening, and the $\ell_2$ distance between the fitted values obtained with screening and no screening. Group screening metrics: the number of groups in $\mathcal{A}_g, \mathcal{S}_g, \mathcal{E}_g$, and $\mathcal{K}_g$. The results are shown for different values of $p$, averaged across $\rho \in \{0,0.3,0.6,0.9\}$. Standard errors are shown.}
\vskip 0.15in
\begin{center}
\begin{small}
\begin{sc}
\begin{tabular}{lrrrrrrrrrr}
\toprule
Method &$p$&\multicolumn{1}{p{1.5cm}}{\raggedleft Runtime \\ screen (s)} & \multicolumn{1}{p{1.5cm}}{\raggedleft Runtime \\ no screen (s)} &$\card(\mathcal{A}_g)$ & $\card(\mathcal{S}_g)$ &  $\card(\mathcal{E}_g)$& $\card(\mathcal{K}_g)$ &\multicolumn{1}{p{1.5cm}}{\raggedleft Num it \\ screen} & \multicolumn{1}{p{1.5cm}}{\raggedleft Num it \\ no screen} & \multicolumn{1}{p{1.5cm}}{\raggedleft $\ell_2$ dist \\ to no screen} \\
\midrule
gSLOPE &$ 500 $&$ 89\pm1 $&$ 144\pm1 $&$ 9\pm0.2 $&$ 10\pm0.3 $&$ 10\pm0.3 $&$ 0.005\pm0.005 $&$ 47\pm1 $&$ 55\pm1 $&$ 6\times 10^{-6}\pm7\times 10^{-6}$\\
  gSLOPE &$ 1625 $&$ 231\pm5 $&$ 453\pm5 $&$ 26\pm1 $&$ 36\pm1 $&$ 36\pm1 $&$ 0.005\pm0.006 $&$ 203\pm6 $&$ 222\pm6 $&$ 1\times 10^{-6}\pm1\times 10^{-6}$  \\
  gSLOPE &$ 2750 $&$ 1061\pm15 $&$ 2296\pm25 $&$ 56\pm2 $&$ 75\pm2 $&$ 75\pm2 $&$ 0.004\pm0.005 $&$ 270\pm9 $&$ 350\pm9 $&$ 5\times 10^{-7}\pm7\times 10^{-7}$  \\
  gSLOPE &$ 3875 $&$ 1765\pm83 $&$ 2800\pm113 $&$ 82\pm3 $&$ 114\pm3 $&$ 114\pm3 $&$ 0.006\pm0.008 $&$ 549\pm19 $&$ 546\pm18 $&$ 3\times 10^{-7}\pm5\times 10^{-7}$  \\
  gSLOPE &$ 5000 $&$ 1937\pm63 $&$ 2422\pm66 $&$ 102\pm3 $&$ 147\pm4 $&$ 147\pm4 $&$ 0.007\pm0.012 $&$ 594\pm21 $&$ 581\pm19 $&$ 2\times 10^{-7}\pm3\times 10^{-7} $ \\
  SGS &$ 500 $&$ 94\pm1 $&$ 133\pm2 $&$ 9\pm0.2 $&$ 11\pm0.3 $&$ 12\pm0.3 $&-&$ 28\pm1 $&$ 74\pm3 $& $5\times 10^{-10}\pm4\times 10^{-10}$  \\
  SGS &$ 1625 $&$ 416\pm8 $&$ 1129\pm19 $&$ 25\pm1 $&$ 37\pm1 $&$ 41\pm1 $&-&$ 62\pm7 $&$ 511\pm21 $& $2\times 10^{-9}\pm2\times 10^{-9}$  \\
  SGS &$ 2750 $&$ 639\pm14 $&$ 2137\pm47 $&$ 62\pm2 $&$ 82\pm2 $&$ 89\pm3 $&-&$ 80\pm6 $&$ 791\pm28 $& $2\times 10^{-9}\pm6\times 10^{-9}$  \\
  SGS &$ 3875 $&$ 939\pm31 $&$ 2862\pm96 $&$ 93\pm3 $&$ 124\pm3 $&$ 136\pm4 $&-&$ 112\pm8 $&$ 1049\pm34 $& $5\times 10^{-9}\pm1\times 10^{-8}$  \\
  SGS &$ 5000 $&$ 1586\pm66 $&$ 2891\pm128 $&$ 119\pm4 $&$ 164\pm3 $&$ 180\pm4 $&-&$ 171\pm11 $&$ 1118\pm37 $& $2\times 10^{-9}\pm4\times 10^{-9}$ \\

\bottomrule
\end{tabular}
\label{tbl:full_table_grp_linear_p}
\end{sc}
\end{small}
\end{center}
\vskip -0.1in
\end{table}
\end{sidewaystable}

\clearpage 
\newpage

\subsubsection{Additional Results for the Logistic Model}\label{appendix:log_model_plot}
This section presents additional results for the logistic model. Similar trends to the ones observed for the linear model are seen. 

\begin{figure}[H]
\vskip 0.2in
\begin{center}
\centerline{\includegraphics[width=0.7\columnwidth]{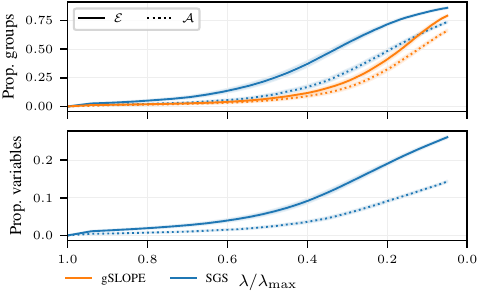}}
\caption{The number of groups/variables in $\mathcal{E}, \mathcal{A}$ as a function of the regularization path for the logistic model with $p=2750, \rho=0.6, m = 197$, shown for gSLOPE and SGS. The results are averaged over $100$ repetitions, with $95\%$ confidence intervals shown. This figure is the equivalent of Figure \ref{fig:path_plot_gslope_sgs} for the logistic model.}
\label{fig:appendix_log_org}
\end{center}
\vskip +0.2in
\end{figure}

\begin{figure}[H]
  \begin{subfigure}[t]{.49\textwidth}
    \centering
    \includegraphics[width=\linewidth]{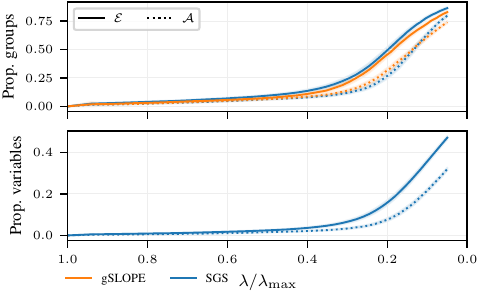}
    \caption{$p=500, \rho = 0$}
  \end{subfigure}
  \hfill
  \begin{subfigure}[t]{.49\textwidth}
    \centering
    \includegraphics[width=\linewidth]{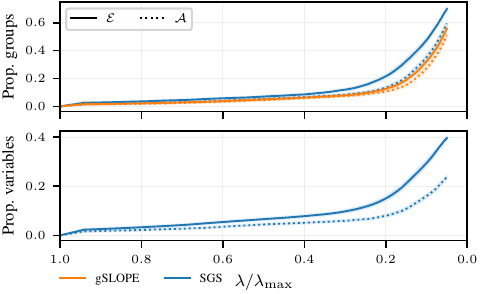}
    \caption{$p=500, \rho = 0.9$}
  \vspace*{2mm}
  \end{subfigure}
  \begin{subfigure}[t]{.49\textwidth}
    \centering
    \includegraphics[width=\linewidth]{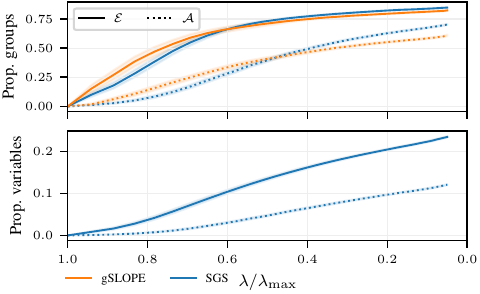}
    \caption{$p=5000, \rho = 0$}
  \end{subfigure}
  \hfill
  \begin{subfigure}[t]{.49\textwidth}
    \centering
    \includegraphics[width=\linewidth]{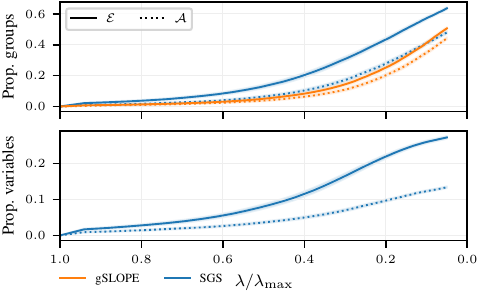}
    \caption{$p=5000, \rho = 0.9$}
  \end{subfigure}
  \caption{The number of groups/variables in $\mathcal{E}, \mathcal{A}$ as a function of the regularization path for the logistic model with SGS and gSLOPE, shown for different values of the correlation ($\rho$) and $p$. The results are averaged over $100$ repetitions, with $95\%$ confidence intervals shown. This figure is the equivalent of Figure \ref{fig:appendix_linear} for the logistic model.}
  \label{fig:appendix_log}
\end{figure}
\begin{figure}[H]
\vskip 0.2in
\begin{center}
\centerline{\includegraphics[width=.7\columnwidth]{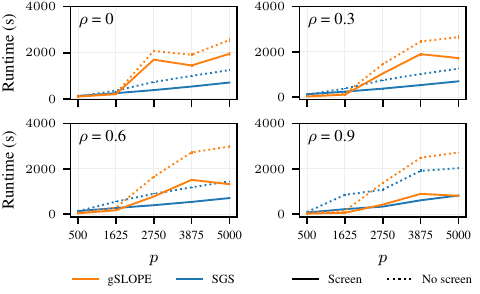}}
\caption{Runtime (in seconds) for screening against no screening as a function of $p$, broken down into different correlation cases, for the logistic model. The results are averaged over $100$ repetitions, with standard errors shown. This figure is the equivalent of Figure \ref{fig:time_fcn_of_p} for the logistic model.}
\label{fig:time_fcn_of_p_log}
\end{center}
\vskip -0.2in
\end{figure}

\begin{sidewaystable}
\begin{table}[H]
\caption{Variable screening metrics for SGS using a logistic model for the simulation study presented in Section \ref{section:results_sim}. The number of variables in $\mathcal{A}_v, \mathcal{S}_v, \mathcal{E}_v$, and $\mathcal{K}_v$ are shown. The results are shown for different values of $p$, averaged across $\rho \in \{0,0.3,0.6,0.9\}$. Standard errors are shown.}
\vskip 0.15in
\begin{center}
\begin{small}
\begin{sc}
\begin{tabular}{lrrrrr}
\toprule
 Method & $p$ & $\card(\mathcal{A}_v)$ & $\card(\mathcal{S}_v)$ &  $\card(\mathcal{E}_v)$& $\card(\mathcal{K}_v)$\\
\midrule
SGS &$ 500 $&$ 53\pm2 $&$ 71\pm3 $&$ 89\pm4 $&$ 19\pm1$  \\
  SGS &$ 1625 $&$ 157\pm5 $&$ 248\pm8 $&$ 291\pm9 $&$ 44\pm1$ \\
  SGS &$ 2750 $&$ 247\pm7 $&$ 420\pm11 $&$ 491\pm13 $&$ 71\pm2$  \\
  SGS &$ 3875 $&$ 316\pm9 $&$ 571\pm14 $&$ 663\pm16 $&$ 92\pm3$  \\
  SGS &$ 5000 $&$ 375\pm10 $&$ 717\pm16 $&$ 824\pm19 $&$ 107\pm3$  \\
\bottomrule
\end{tabular}
\label{tbl:full_table_var_log}
\end{sc}
\end{small}
\end{center}
\vskip -0.1in
\end{table}
\vspace{20pt}
\begin{table}[H]
\caption{General and group screening metrics for SGS and gSLOPE using logistic models for the simulation study presented in Section \ref{section:results_sim}. General metrics: the runtime (in seconds) for screening against no screening, the number of fitting iterations for screening against no screening, and the $\ell_2$ distance between the fitted values obtained with screening and no screening. Group screening metrics: the number of groups in $\mathcal{A}_g, \mathcal{S}_g, \mathcal{E}_g$, and $\mathcal{K}_g$. The results are shown for different values of $p$, averaged across $\rho \in \{0,0.3,0.6,0.9\}$. Standard errors are shown.}
\vskip 0.15in
\begin{center}
\begin{small}
\begin{sc}
\begin{tabular}{lrrrrrrrrrr}
\toprule
Method &$p$&\multicolumn{1}{p{1.5cm}}{\raggedleft Runtime \\ screen (s)}&\multicolumn{1}{p{1.5cm}}{\raggedleft Runtime \\ no screen (s)}&$\card(\mathcal{A}_g)$ & $\card(\mathcal{S}_g)$ &  $\card(\mathcal{E}_g)$& $\card(\mathcal{K}_g)$ &\multicolumn{1}{p{1.5cm}}{\raggedleft Num it \\ screen} & \multicolumn{1}{p{1.5cm}}{\raggedleft Num it \\ no screen} & \multicolumn{1}{p{1.5cm}}{\raggedleft $\ell_2$ dist \\ to no screen} \\
\midrule
gSLOPE &$ 500 $&$ 49\pm1 $&$ 76\pm1 $&$ 26\pm1 $&$ 31\pm1 $&$ 31\pm1 $&$ 0.003\pm0.004 $&$ 31\pm1 $&$ 40\pm1 $&$ 4\times 10^{-8}\pm5\times 10^{-8}$  \\
  gSLOPE &$ 1625 $&$ 138\pm3 $&$ 203\pm3 $&$ 43\pm1 $&$ 54\pm2 $&$ 54\pm2 $&$ 0.004\pm0.006 $&$ 78\pm2 $&$ 79\pm1 $&$ 1\times 10^{-8}\pm5\times 10^{-9} $\\ 
  gSLOPE &$ 2750 $&$ 987\pm11 $&$ 1641\pm16 $&$ 72\pm2 $&$ 95\pm3 $&$ 95\pm3 $&$ 0.008\pm0.013 $&$ 87\pm2 $&$ 89\pm1 $&$ 1\times 10^{-8}\pm5\times 10^{-9}$ \\ 
  gSLOPE &$ 3875 $&$ 1441\pm26 $&$ 2398\pm31 $&$ 98\pm3 $&$ 135\pm3 $&$ 135\pm3 $&$ 0.031\pm0.054 $&$ 95\pm2 $&$ 98\pm1 $&$ 7\times 10^{-9}\pm3\times 10^{-9}$  \\
  gSLOPE &$ 5000 $&$ 1454\pm29 $&$ 2727\pm40 $&$ 118\pm3 $&$ 168\pm4 $&$ 168\pm4 $&$ 0.022\pm0.041 $&$ 101\pm2 $&$ 109\pm1 $&$ 4\times 10^{-9}\pm1\times 10^{-9}$  \\
SGS &$ 500 $&$ 118\pm1 $&$ 113\pm1 $&$ 28\pm1 $&$ 33\pm1 $&$ 38\pm1 $&-&$ 6\pm0.3 $&$ 29\pm1 $&$ 8\times 10^{-9}\pm7\times 10^{-10} $ \\
  SGS &$ 1625 $&$ 248\pm2 $&$ 538\pm9 $&$ 50\pm2 $&$ 59\pm2 $&$ 64\pm2 $&-&$ 7\pm1 $&$ 63\pm3 $&$ 5\times 10^{-9}\pm1\times 10^{-9}$  \\
  SGS &$ 2750 $&$ 374\pm2 $&$ 868\pm12 $&$ 85\pm3 $&$ 104\pm2 $&$ 115\pm3 $&-&$ 7\pm0.4 $&$ 57\pm2 $&$ 3\times 10^{-9}\pm5\times 10^{-10} $ \\
  SGS &$ 3875 $&$ 558\pm4 $&$ 1280\pm19 $&$ 116\pm3 $&$ 148\pm3 $&$ 164\pm3 $&-&$ 8\pm0.4 $&$ 54\pm1 $&$ 2\times 10^{-9}\pm2\times 10^{-10}$  \\
  SGS &$ 5000 $&$ 737\pm5 $&$ 1498\pm19 $&$ 141\pm4 $&$ 188\pm3 $&$ 209\pm4 $&-&$ 8\pm0.4 $&$ 54\pm1 $&$ 1\times 10^{-9}\pm2\times 10^{-10}$  \\
\bottomrule
\end{tabular}
\label{tbl:full_table_grp_log}
\end{sc}
\end{small}
\end{center}
\vskip -0.1in
\end{table}
\end{sidewaystable}

\clearpage
\newpage
\subsection{Data Description}\label{appendix:real_data_description}
\begin{itemize}
    \item carbotax: Carbotax study of ovarian tumor growth.
    \begin{itemize}
        \item Response (continuous): Relative tumor volume ($\log_2$ scale).
        \item Data matrix: Gene expression measurements. $10000$ factors were randomly sampled from a collection of $34964$.
        \item Grouping structure: Variables are grouped using k-means clustering \citep{1056489}.
    \end{itemize}
    \item scheetz: Gene expression data in the mammalian eye. 
    \begin{itemize}
        \item Response (continuous): Gene expression measurements for the Trim32 gene.
        \item Data matrix: Gene expression measurements for other genes.
        \item Grouping structure: Variables are grouped using k-means clustering \citep{1056489}.
    \end{itemize}
  \item adenoma: Transcriptome profile data to identify formation of colorectal adenomas. 
    \begin{itemize}
        \item Response (binary): Binary labels for whether sample came from adenoma or normal mucosa.
        \item Data matrix: Transcriptome profile measurements.
        \item Grouping structure: Genes are assigned to pathways (groups) using the C3 regulatory target gene sets.\footnote{\label{note1}\url{gsea-msigdb.org/gsea/msigdb/human/collections.jsp}. Accessed 08/2024.}
    \end{itemize}
    \item cancer: Breast cancer patients treated with tamoxifen for 5 years.
    \begin{itemize}
        \item Response (binary): Binary labels classifying whether the cancer had recurred.
        \item Data matrix: Gene expression data.
        \item Grouping structure: Genes are assigned to pathways (groups) using the C3 regulatory target gene sets.\footnotemark[1]
    \end{itemize}
    \item celiac: Gene expression data of primary leucocytes to classify celiac disease.
    \begin{itemize}
        \item Response (binary): Binary labels as to whether a patient has celiac disease.
        \item Data matrix: Gene expression measurements from the primary leucocytes.
        \item Grouping structure: Genes are assigned to pathways (groups) using the C3 regulatory target gene sets.\footnotemark[1]
    \end{itemize}
    \item colitis: Blood cells data for classifying whether a patient has colitis.
    \begin{itemize}
        \item Response (binary): Binary labels classifying whether a patient has colitis.
        \item Data matrix: Gene expression measurements.
        \item Grouping structure: Genes are assigned to pathways (groups) using the C3 regulatory target gene sets.\footnotemark[1]
    \end{itemize}
\item tumour: Gene expression data of pancreative cancer samples to identify tumorous tissue. 
    \begin{itemize}
        \item Response (binary): Binary labels indicating if a sample is from tumour tissue.
        \item Data matrix: Gene expression measurements.
        \item Grouping structure: Genes are assigned to pathways (groups) using the C3 regulatory target gene sets.\footnotemark[1]
    \end{itemize}
\end{itemize}
\begin{table}[H]
  \caption{Dataset information for the six datasets used in the real data analysis.}
  \label{tbl:appendix_real_dataset_info}
  \centering
 \resizebox{\textwidth}{!}{ \begin{tabular}{llllllll}
    \toprule
     Dataset & $p$ & $n$ & $m$ & Group sizes & Type & Source \\
    \midrule
carbotax& $10000$ & $101$ & $100$&$[1,126]$& Linear&\cite{carbotax}\tablefootnote{downloaded from \url{https://iowabiostat.github.io/data-sets/}\label{reference footnote}. Accessed 08/2024.} \\
scheetz&  $18975$ & $120$ & $379$&$[1,165]$& Linear & \cite{scheetz2006RegulationDisease}\textsuperscript{\getrefnumber{reference footnote}}  \\
adenoma& $17661$ & $64$ & $1849$&$[1,646]$& Logistic&\cite{SabatesBellver2007TranscriptomeAdenomas}\tablefootnote{downloaded from \url{https://www.ncbi.nlm.nih.gov/}\label{genetics footnote}. Accessed 09/2024.} \\
cancer&$7057$ & $60$ & $1277$&$[1,292]$& Logistic&\cite{Ma2004ATamoxifen}\textsuperscript{\getrefnumber{genetics footnote}} \\
celiac&$14294$ & $132$ & $1666$ & $[1,570]$ & Logistic&\cite{Heap2009ComplexLeucocytes}\textsuperscript{\getrefnumber{genetics footnote}}\\
colitis&$11999$ & $127$ & $1528$ & $[1,497]$ & Logistic&\cite{Burczynski2006MolecularCells}\textsuperscript{\getrefnumber{genetics footnote}}\\
tumour&$17661$ & $52$ & $1849$ & $[1,646]$ & Logistic&\cite{Pei2009FKBP51Akt,Ellsworth2013ContributionAdenocarcinoma,Li2016GeneticCancer}\textsuperscript{\getrefnumber{genetics footnote}}\\
    \bottomrule
  \end{tabular}
  }
\end{table}

\subsection{Additional Results from the Real Data Analysis}
\begin{figure}[H]
 \centering
  \includegraphics[width=.6\linewidth,valign=t]{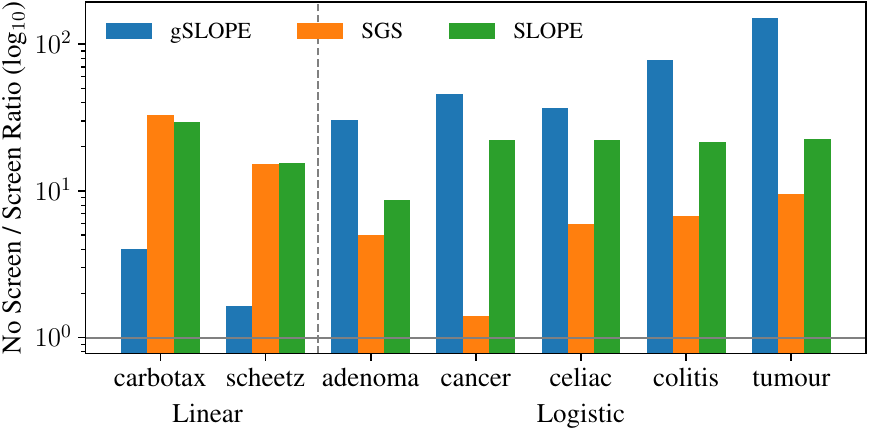}
  \caption{The ratio of no screen time to screen time ($\uparrow$) of SLOPE, gSLOPE, and SGS applied to the real datasets, for fitting $100$ path models, split into response type. The horizontal grey line represents no screening improvement.}
  \label{fig:real_bar_chart_slope}

\end{figure}

\begin{figure}[H]
 \centering
  \includegraphics[width=.6\linewidth,valign=t]{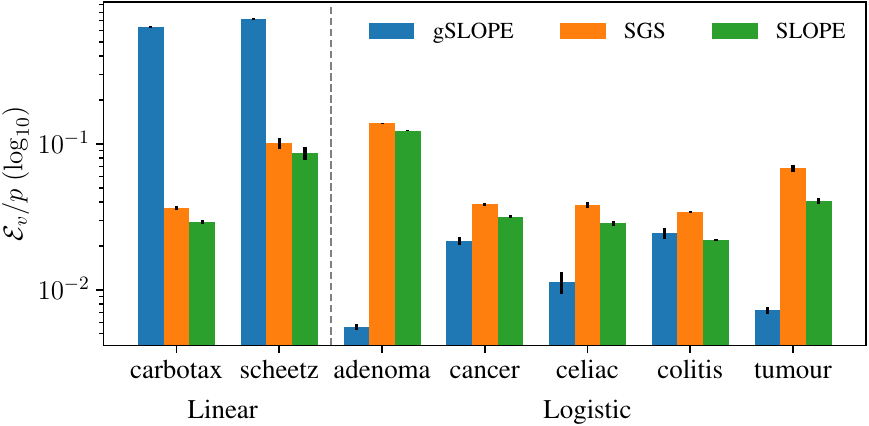}
  \caption{The ratio of the fitting set ($\mathcal{E}_v$) to the input dimensionality ($p$) ($\downarrow$) of SLOPE, gSLOPE, and SGS applied to the real datasets, for fitting $100$ path models, split into response type.}
  \label{fig:real_data_propn}

\end{figure}

\begin{figure}[H]
 \centering
\includegraphics[width=.6\linewidth,valign=t]{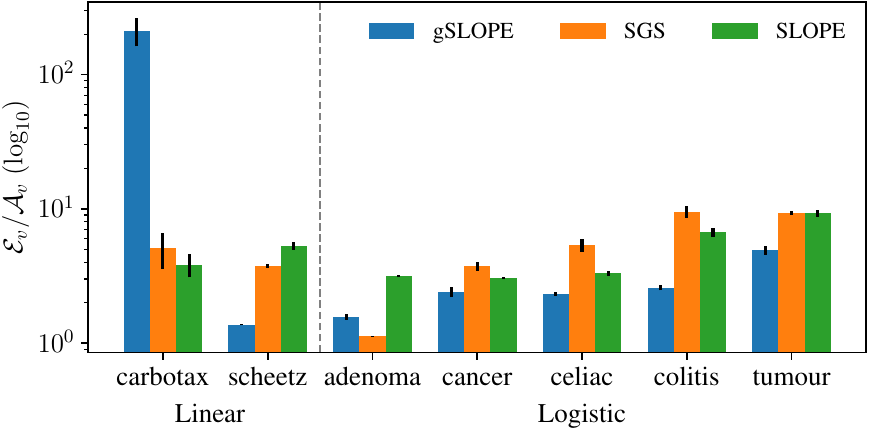}
  \caption{The ratio of the fitting set ($\mathcal{E}_v$) to the active set ($\mathcal{A}_v$) ($\downarrow$) of SLOPE, gSLOPE, and SGS applied to the real datasets, for fitting $100$ path models, split into response type.}
  \label{fig:real_data_propn_active}

\end{figure}
\label{appendix:real_data_results}

\begin{figure}[H]
    \centering
    \includegraphics[width=.6\linewidth]{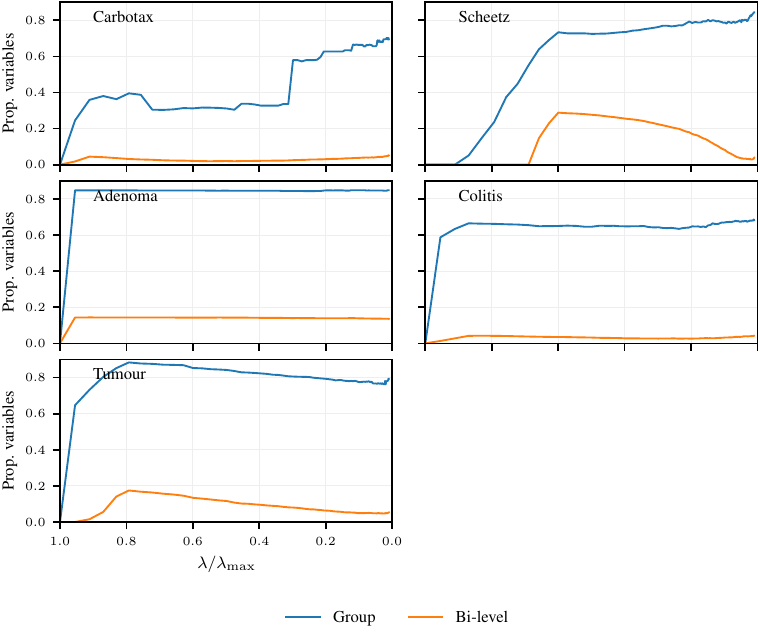}
    \captionsetup{width=.95\linewidth}
    \caption{The proportion of variables in $\mathcal{S}_v$ relative to $p$ for group-only and bi-level screening applied to SGS, plotted along the regularization path for the carbotax, scheetz, adenoma, colitis, and tumour datasets.}
\label{fig:bilevelcomparison_all}
\end{figure}
\begin{sidewaystable}
\begin{table}[H]
\caption{Variable screening metrics for SGS applied to real data in Section \ref{section:results_real}. The number of variables in $\mathcal{A}_v, \mathcal{S}_v, \mathcal{E}_v$, and $\mathcal{K}_v$ are shown. The results are averaged across the nine pathway collections, with standard errors shown.}
\vskip 0.15in
\begin{center}
\begin{small}
\begin{sc}
\begin{tabular}{llrrrr}
\toprule
Method & Dataset & $\card(\mathcal{A}_v)$ & $\card(\mathcal{S}_v)$ &  $\card(\mathcal{E}_v)$& $\card(\mathcal{K}_v)$\\
\midrule
SGS&carbotax &   $144\pm6$& $361\pm10$  & $364\pm10$ & $2\pm0.50$\\
SGS&scheetz &  $612\pm74$& $1908\pm174$  & $1915\pm175$ & $7\pm1.28$\\
SGS&adenoma &  $2174\pm23$& $2451\pm4$  & $2463\pm5$ & $12\pm0.22$\\
SGS&cancer &   $86\pm3$& $273\pm7$  & $273\pm7$ & $1\pm0.05$\\
SGS&celiac &   $175\pm10$& $545\pm22$  & $546\pm22$ & $1\pm0.07$\\
SGS&colitis &   $68\pm4$& $409\pm7$  & $411\pm7$ & $2\pm0.18$\\
SGS&tumour &   $128\pm5$& $1189\pm59$  & $1200\pm59$ & $12\pm0.54$\\

\bottomrule
\end{tabular}
\label{tbl:full_table_var_real_data}
\end{sc}
\end{small}
\end{center}
\vskip -0.1in
\end{table}
\vspace{20pt}
\begin{table}[H]
\caption{General and group screening metrics for SGS and gSLOPE applied to real data in Section \ref{section:results_real}. General metrics: the runtime (in seconds) for screening against no screening, the number of fitting iterations for screening against no screening (with the number of occasions of failed convergence given in brackets), and the $\ell_2$ distance between the fitted values obtained with screening and no screening. Group screening metrics: the number of groups in $\mathcal{A}_g, \mathcal{S}_g, \mathcal{E}_g$, and $\mathcal{K}_g$. The results are averaged across the nine pathway collections, with standard errors shown.}
\vskip 0.15in
\begin{center}
\begin{small}
\begin{sc}
\begin{tabular}{llrrrrrrrrrr}
\toprule
Method & Dataset &\multicolumn{1}{p{1.5cm}}{\raggedleft Runtime \\ screen (s)} & \multicolumn{1}{p{1.5cm}}{\raggedleft Runtime \\ no screen (s)} & $\card(\mathcal{A}_g)$ & $\card(\mathcal{S}_g)$ &  $\card(\mathcal{E}_g)$& $\card(\mathcal{K}_g)$ &\multicolumn{1}{p{1.5cm}}{\raggedleft Num it \\ screen (num failed)} & \multicolumn{1}{p{1.5cm}}{\raggedleft  Num it \\ no screen (num failed)} & \multicolumn{1}{p{1.5cm}}{\raggedleft $\ell_2$ dist \\ to no screen} \\
\midrule
gSLOPE&carbotax&$865$&$3456$&$ 71\pm3 $&$ 124\pm2 $&$124\pm2 $&$0.05\pm 0.05$& $2367\pm176(1)$&$6433\pm303(10)$&$1\times10^{-9}\pm2\times10^{-10}$\\
gSLOPE&scheetz&$4731$&$ 7735$&$ 153\pm6$&$ 212\pm6 $&$212\pm6 $&$0\pm 0$ &$6252\pm313(20)$&$7183\pm332(38)$&$8\times10^{-10}\pm8\times10^{-10}$\\
gSLOPE&adenoma& $415$&$ 12610$&$71\pm6$ &$77\pm4$&$77\pm4$&$0\pm 0$ &$3956\pm378(14)$&$7864\pm394(78)$&$9\times10^{-6}\pm1\times10^{-6}$\\
gSLOPE&cancer&$ 109$&$ 5003$&$ 46\pm3$&$81\pm5$&$81\pm5$&$0\pm 0$ &$764\pm74(0)$&$6250\pm404(40)$&$4\times10^{-7}\pm8\times10^{-8}$\\
gSLOPE&celiac&$ 205$&$ 7530$&$ 34\pm75$&$62\pm9$&$62\pm9$&$0\pm 0$ &$834\pm177(0)$&$4764\pm428(27)$&$4\times10^{-6}\pm1\times10^{-5}$\\
gSLOPE&colitis&$148$&$11567$&$69\pm5 $&$114\pm8$&$114\pm8$&$0\pm 0$ &$749\pm62(0)$&$9090\pm206(72)$&$4\times10^{-6}\pm7\times10^{-7}$\\
gSLOPE&tumour&$222$&$33405$&$ 24\pm1$&$83\pm4$&$83\pm4$&$0\pm 0$ &$892\pm67(0)$&$9224\pm218(86)$&$4\times10^{-7}\pm5\times10^{-8}$\\
SGS&carbotax&$63$&$2067$&$ 40\pm1 $&$ 138\pm1 $&$119\pm1 $&- & $168\pm60(0)$&$3572\pm350 (3)$&$8\times10^{-10}\pm3\times10^{-10}$\\
SGS&scheetz&$358$&$ 5476$&$ 105\pm4 $&$ 203\pm5 $&$186\pm6 $&- &$10\pm4(0)$&$5039\pm360 (28)$&$7\times10^{-8}\pm2\times10^{-8}$\\
SGS&adenoma& $1234$&$ 6143$&$601\pm6$ &$789\pm1$&$735\pm1$&- &$1172\pm69(0)$&$4224\pm338 (20)$&$6\times10^{-7}\pm5\times10^{-8}$\\
SGS&cancer&$ 151$&$ 211$&$ 70\pm2$&$295\pm3$&$182\pm4$&- &$124\pm15(0)$&$147\pm14(0)$&$2\times10^{-8}\pm2\times10^{-9}$\\
SGS&celiac&$ 339$&$ 2035$&$ 120\pm7$&$460\pm10$&$288\pm10$&- &$55\pm12(0)$&$1259\pm112(0)$&$9\times10^{-5}\pm2\times10^{-5}$\\
SGS&colitis&$271$&$1821$&$52\pm3 $&$368\pm3$&$230\pm3$&- &$54\pm10(0)$&$1411\pm188 (0)$&$1\times10^{-6}\pm3\times10^{-7}$\\
SGS&tumour&$781$&$7451$&$ 94\pm3$&$733\pm10$&$479\pm14$&- &$33\pm7(0)$&$6006\pm287 (23)$&$7\times10^{-7}\pm2\times10^{-7}$\\
\bottomrule
\end{tabular}
\label{tbl:full_table_grp_real_data}
\end{sc}
\end{small}
\end{center}
\vskip -0.1in
\end{table}
\end{sidewaystable}

\end{document}

%% file: new_sections/1-introduction_v2.tex
As the amount of data collected increases, the emergence of high-dimensional data, where the number of features ($p$) is much larger than the number of observations ($n$), is becoming increasingly common in fields ranging from genetics to finance. Performing regression and discovering relevant features on these datasets is a challenging task, as classical statistical methods tend to break down. The most popular approach to meeting this challenge is the lasso \citep{Tibshirani1996RegressionLasso}, which has given rise to the general penalized regression framework
\begin{equation}\label{eqn:general_problem}
    \hat\beta(\lambda) \in \argmin_{\beta \in \mathbb{R}^p} \left\{ f(\beta)+\lambda J(\beta;v)\right\},
\end{equation}
where $f$ is a differentiable and convex loss function, $J$ is a convex penalty norm, $v$ are penalty weights, and $\lambda>0$ is the regularization parameter. 

A key aspect of fitting a penalized model is to tune the value of $\lambda$ along an $l$-length path $\lambda_1\geq \ldots \geq \lambda_l\geq 0$. Several approaches exist for tuning this parameter, including cross-validation \citep{Homrighausen,lassocv} and exact solution path algorithms \citep{LARS}, but these can be computationally expensive. Screening rules help alleviate these costs by discarding variables that are inactive at the optimal solution, thus reducing input dimensionality prior to optimization. 

Denote the \textit{active set} of coefficients at path point $\lambda_{k+1}$, for $k\in [l-1]:=\{1,\ldots,l-1\}$, by $\mathcal{A}_v(\lambda_{k+1}) = \{i \in [p]:\hat\beta_i(\lambda_{k+1})\neq 0\}$. The goal of a (sequential) screening rule is to use the solution at $\lambda_k$ to recover a \textit{screened set} of features, $\mathcal{S}_v(\lambda_{k+1})$, which is a superset of $\mathcal{A}_v(\lambda_{k+1})$. The screened set is then used as input for calculating the fitted values, leading to significant computational savings.

There are two types of screening rules: \textit{safe} and \textit{heuristic}. Safe rules are guaranteed to only discard inactive variables and mostly follow the Safe Feature Elimination (SAFE) framework \citep{Ghaoui2010SafeProblems}, in which safe regions for variables are constructed. Other notable examples include Slores \citep{Wang2014ARegression}, the dome test \citep{Xiang2012FastCorrelations}, and Dual Polytope Projections (DPP) \citep{Wang2013LassoProjection}, as well as sample screening \citep{Shibagaki2016} and other examples given in \cite{pmlr-v28-ogawa13b, Fercoq2015MindLasso,pmlr-v119-atamturk20a}. 

Heuristic rules tend to follow the strong screening rule framework, proposed by \cite{tibshirani2010strong}, and discard considerably more variables than safe rules. However, they can incorrectly discard active variables, so are complemented by a check of the Karush–Kuhn–Tucker (KKT) \citep{Kuhn1950NonlinearProgramming} stationarity conditions. Other strong rules include Blitz \citep{pmlr-v37-johnson15}, SIS \citep{Fan2008}, and ExSIS \citep{AHMED201933}. Hybrid schemes exist, using both safe and heuristic rules \citep{Zeng2021, Wang2022}.

A strong screening rule is formulated through the KKT stationarity conditions for Equation \ref{eqn:general_problem}, given by
\begin{equation}\label{eqn:kkt_condition}
    \mathbf{0} \in \nabla f(\beta) + \lambda\partial J(\beta;v).
\end{equation}
If the gradient were available, the active set could be identified exactly by checking the subdifferential of the norm at zero:
\begin{equation}\label{eqn:subdiff_at_zero}
    \nabla f(\beta) \in \lambda \partial J(\mathbf{0};v) = \{x\in\mathbb{R}^p : J^*(x;\lambda v) \leq 1\},
\end{equation}
where $J^*$ is the dual norm of $J$ and $\partial J(\mathbf{0};v)$ is the unit ball of the dual norm. So, $J^*(\nabla f(\beta);\lambda v)\leq 1$ indicates that $\beta=0$. As the gradient at $\lambda_{k+1}$ is not available, a model-specific approximation of the gradient is derived to find a screened subset of the features, $\mathcal{S}_v(\lambda_{k+1})$, such that $\mathcal{A}_v(\lambda_{k+1})\subset\mathcal{S}_v(\lambda_{k+1})$. 


\subsection{Screening Approaches for SLOPE} 
As the lasso is inconsistent under certain scenarios \citep{Zou2006},  several adaptive extensions have been proposed, including the Sorted L-One Penalized Estimation (SLOPE) model \citep{Bogdan2015SLOPEAdaptiveOptimization}. SLOPE applies the sorted $\ell_1$ norm $J_\text{slope}(\beta;v) = \sum_{i=1}^p v_i|\beta|_{(i)}$, where $v_1 \geq \ldots \geq v_p \geq 0, |\beta|_{(1)}\geq \ldots \geq |\beta|_{(p)}$. One key advantage of SLOPE is its ability to control the variable false discovery rate (FDR) under orthogonal data. Additional powerful properties include: it clusters strongly correlated features, it finds the minimum total squared error loss across different sparsity levels, removing the need for prior knowledge of sparsity, and it is asymptotically minimax \citep{figueiredo2014sparseestimationstronglycorrelated,SuCandes}. All of these useful properties have meant that SLOPE has found widespread use in machine learning and genetics \citep{Gossmann2015,virouleau2017highdimensionalrobustregressionoutliers,KREMER2020105687,10.1371/journal.pone.0269369,riccobello2023sparsegraphicalmodellingsorted}.

Both safe \citep{Bao2020FastRules,Elvira2021SafeProblem} and strong \citep{Larsson2020a} rules have been proposed for SLOPE, as well as exact solution path algorithms \citep{Nomura2020,Dupuis2023}. As SLOPE is a non-separable penalty, safe screening requires repeated screening during optimization, which is expensive due to the repeated dual norm evaluations required for the safe regions \citep{Larsson2020a}.

\paragraph{Group-based SLOPE Models}
In genetics, the analysis of grouped features is frequently encountered, as genes are grouped into pathways for the completion of a specific biological task. To use this grouping information, SLOPE has been extended to group (gSLOPE) and sparse-group (SGS) regression.

For a set of $m$ non-overlapping groups, $\mathcal{G}_1,\dots,\mathcal{G}_m$  of sizes $p_1,\ldots, p_m$, Group SLOPE (gSLOPE) \citep{Gossmann2015,Brzyski2019GroupPredictors} is given by
\begin{equation}
    J_\text{gslope}(\beta;w) = \sum_{g=1}^m \sqrt{p_g} w_g \|\beta^{(g)}\|_2,
\end{equation}
such that $\beta^{(g)}\in\mathbb{R}^{p_g}$ is a vector of the group coefficients. The norm has ordered penalty weights $w_1\geq \ldots \geq w_g \geq 0$ (described in Appendix \ref{appendix:gslope_pen_seq}) which are matched to $\sqrt{p_1}\|\beta^{(1)}\|_2 \geq \ldots\geq \sqrt{p_m}\|\beta^{(m)}\|_2$.

Sparse-group SLOPE (SGS) \citep{Feser2023Sparse-groupFDR-control} was further proposed as a convex combination of SLOPE and gSLOPE for concurrent variable and group selection. For $\alpha \in [0,1]$, with weights $(v,w)$ (described in Appendix \ref{appendix:sgs_pen_seq}), the norm is given by
\begin{equation*}
    J_\text{sgs}(\beta;\alpha,v,w) = \alpha J_\text{slope}(\beta; v) +  (1-\alpha) J_\text{gslope}(\beta; w).
\end{equation*}
Both approaches control the FDR under orthogonal data: gSLOPE at the group-level \citep{Brzyski2019GroupPredictors} and SGS at both levels \citep{Feser2023Sparse-groupFDR-control}. SGS has been found to outperform other methods at selection and prediction \citep{Feser2023Sparse-groupFDR-control}.

\subsection{Contributions} 
No screening rules exist for group-based SLOPE models. The strong screening rule framework introduced by \cite{tibshirani2010strong} facilitated the extension from the lasso to the group lasso by requiring only knowledge of the dual norm. However, this framework is restricted to separable penalties, making it unsuitable for SLOPE. Instead, in \cite{Larsson2020a}, the subdifferential of SLOPE is used to derive strong rules.

Motivated by this, we propose a new strong screening framework for sparse-group norms (Section \ref{sec:framework}), used to develop screening for gSLOPE (Section \ref{sec:gSLOPE}) and SGS (Section \ref{sec:SGS}) (with proofs of the results provided in Appendices \ref{appendix:gslope_theory} and \ref{appendix:sgs_theory}). The framework applies two screening layers, drastically reducing dimensionality (Figure \ref{fig:bi-level-screening}). The screening requires knowledge of the subdifferentials, necessitating a general derivation for gSLOPE (Theorem \ref{thm:gslope_subdiff}).

The choice of strong screening over safe is motivated by two main reasons. First, strong rules discard significantly more variables than safe rules \citep{tibshirani2010strong}. Second, safe rules require calculating the dual norm set. For SGS, this is a sum of convex sets. Determining if a point lies in this set requires knowledge of the summation's decomposability, which is a difficult task \citep{Wang2014Two-LayerSets}. This challenge is addressed in Section \ref{section:sgs_group_screen}. 

The analysis of synthetic and real datasets shows that our proposed screening rules considerably improve runtime (Section \ref{section:results}). The reduced input dimensionality from the screening also eases convergence issues with large datasets. These improvements are achieved without affecting solution optimality.

\begin{figure}[t]
    \centering
    \includegraphics[width=1\linewidth]{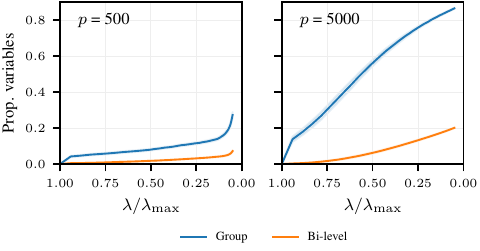}
    \caption{The proportion of variables in $\mathcal{S}_v$ relative to $p$ for group-only and bi-level screening applied to SGS, plotted along the regularization path with 95\% confidence intervals. Synthetic data was generated under a linear model for $p = 500, 5000$ (Section \ref{section:results_sim}), with results averaged over 100 repetitions.}
        \label{fig:bi-level-screening}
\end{figure} 

As the proposed screening rules only require that the penalty sequences ($v, w$) are ordered, they apply to the wider class of group-based Ordered Weighted $\ell_1$ (OWL) models. Popular special cases of OWL models include SLOPE and the Octagonal Shrinkage and Clustering Algorithm for Regression (OSCAR) model \citep{oscar} (the description of group-based OSCAR models is provided in Section \ref{section:owl}).



\paragraph{Notation}
The sets of active and inactive groups are given  by $\mathcal{A}_g = \{g \in [m]: \|\hat\beta^{(g)}\|_2\neq 0\}$ and $\mathcal{Z} = \{g \in [m]: \|\hat\beta^{(g)}\|_2= 0\}$. Their corresponding set of variable indices are denoted by $\mathcal{G}_\mathcal{A}$ and $\mathcal{G}_\mathcal{Z}$. The cardinality of a vector $x$ is denoted by $\card(x)$. The operators $(\cdot)_\downarrow$ and $(\cdot)_{|\downarrow|}$ sort a
vector into decreasing and decreasing absolute form. We use $\preceq$ to denote element-wise inequality signs. The operator $\mathcal{O}(\cdot)$ returns the index of a vector sorted into decreasing absolute form. The cumulative summation operator applied on a vector is denoted by $\text{cumsum}(x) = [x_1,x_1+x_2,\ldots, \sum_{i=1}^{\card(x)} x_i]$.

%% file: new_sections/2-sparse-group-framework.tex
Sparse-group models, such as SGS and the sparse-group lasso (SGL) \citep{Simon2013}, apply both variable and group penalization so that bi-level screening is possible. Safe rules that perform bi-level screening exist for SGL \citep{Wang2014Two-LayerSets,Ndiaye2016GAPLasso}, but there are no such strong rules. The strong screening framework \citep{tibshirani2010strong} does not extend to sparse-group or non-separable norms, and the strong rule derived for SGL in \citet{Liang2022Sparsegl:Lasso} applies only group-level screening. 

\paragraph{Framework} We introduce a new sparse-group framework (Algorithm \ref{alg:sgs_framework}), based on the strong framework by \citet{tibshirani2010strong}, for applying strong screening to sparse-group norms, allowing for bi-level screening. By applying bi-level screening for SGS, a substantially larger proportion of variables are discarded than with group screening alone (Figure \ref{fig:bi-level-screening}).

First, a screened set of groups is computed, $\mathcal{S}_g$.  An additional layer of screening is then performed to compute $\mathcal{S}_v$ using $\mathcal{S}_g$. This, combined with the previously active variables, forms the reduced input set for fitting, $\mathcal{E}_v$. KKT checks are performed on $\mathcal{E}_v$ to ensure no violations have occurred (Appendices \ref{appendix:gslope_kkt} and \ref{appendix:sgs_kkt}). 
Based on this framework, the SGS rules are derived in Section \ref{sec:SGS}. The screening for gSLOPE (Section \ref{sec:gSLOPE}) also uses this framework but does not perform the variable screening and instead takes $\mathcal{S}_v$ as all variables in the groups of $\mathcal{S}_g$. The KKT checks are then performed only on the groups. Appendix \ref{appendix:framework} describes the implementation of the framework for gSLOPE and SGS.

The main cost of the framework is calculating the fitted values. For $t$ iterations of ATOS (a proximal algorithm used to fit SGS, see Section \ref{section:results_sim}), the convergence rate is $O(1/t)$ \citep{pmlr-v80-pedregosa18a} and we expect a time complexity of $O(tp^2)$ for a proximal algorithm \citep{https://doi.org/10.1002/wics.1602}.


\vspace{3.6mm}
\begin{algorithm}[H]
   \caption{Sparse-group screening framework}
   \label{alg:sgs_framework}
\begin{algorithmic}
   \STATE {\bfseries Input:}  $(\lambda_1,\ldots,\lambda_l)\in\mathbb{R}^l$, $\mathbf{X}\in \mathbb{R}^{n\times p}, y\in \mathbb{R}^n$
\FOR{$k=1$ {\bfseries to} $l-1$}
\STATE $\mathcal{S}_g(\lambda_{k+1})$ $\leftarrow$ group screening on full input
\STATE $\mathcal{S}_v(\lambda_{k+1})$ $\leftarrow$ variable screening on $g \in \mathcal{S}_g(\lambda_{k+1})$
\STATE $\mathcal{E}_v \leftarrow \mathcal{S}_v(\lambda_{k+1}) \cup \mathcal{A}_v(\lambda_{k})$
\STATE compute $\hat\beta_{\mathcal{E}_v}(\lambda_{k+1})$ 
\STATE $\mathcal{K}_v \leftarrow$ variable KKT violations on $\hat\beta(\lambda_{k+1})$
\WHILE{$ \card(\mathcal{K}_v)> 0$}
\STATE $\mathcal{E}_v \leftarrow \mathcal{E}_v\cup \mathcal{K}_v$
\STATE compute $\hat\beta_{\mathcal{E}_v}(\lambda_{k+1})$
\STATE $\mathcal{K}_v \leftarrow$ variable KKT violations on $\hat\beta(\lambda_{k+1})$
\ENDWHILE
\ENDFOR
 \STATE {\bfseries Output:} $\hat\beta(\lambda_1), \ldots, \hat\beta(\lambda_l) \in \mathbb{R}^{p}$
\end{algorithmic}
\end{algorithm}

%% file: new_sections/3-gslope.tex
\label{sec:gSLOPE}

The strong rule for gSLOPE is formulated by checking the zero condition of the subdifferential (as per Equation \ref{eqn:subdiff_at_zero}) derived in Theorem \ref{thm:gslope_subdiff}. To derive the subdifferential, we define the operator
\begin{equation*}
    [b]_{\mathcal{G},q} := (p_1^{q}\|b^{(1)}\|_2,\dots,p_m^{q}\|b^{(m)}\|_2)^\top. 
\end{equation*}
In particular, $[b]_{\mathcal{G}_\mathcal{Z},-0.5}$ is the operator applied only to the inactive groups using the quotient $q = -0.5$.

\begin{theorem}[gSLOPE subdifferential]\label{thm:gslope_subdiff}
The subdifferential for gSLOPE is given by
\begin{equation*}
        \partial J_\text{gslope}(\beta;w) =  
        \begin{dcases}
        \begin{array}{l}
         \big\{x \in \mathbb{R}^{\card\mathcal{G}_\mathcal{Z}}: \\
         {[x]}_{\mathcal{G}_\mathcal{Z},-0.5} \in \partial J_\text{slope}(0; w_{\mathcal{Z}})\big\},
        \end{array} \text{at 0.}& \\
         \left\{w_g \sqrt{p_g}\frac{\beta^{(g)}}{\|\beta^{(g)}\|_2} \right\}, \;\;\;\;\;\;\;\;\; \text{otherwise.}&
        \end{dcases}
\end{equation*}
\end{theorem}

The choice of $q = -0.5$ leads to $J_\text{gslope}^*(x;w) = J_\text{slope}^*([x]_{\mathcal{G},-0.5})$, which allows the gSLOPE subdifferential to be written in terms of the SLOPE one \citep{Brzyski2019GroupPredictors}.
Combining the KKT conditions at zero (Equation \ref{eqn:subdiff_at_zero}) with the gSLOPE subdifferential (Theorem \ref{thm:gslope_subdiff}) reveals that a group is inactive if $h(\lambda) := ([\nabla f(\hat\beta(\lambda))]_{\mathcal{G},-0.5})_\downarrow  \in \partial J_\text{slope}(\mathbf{0};\lambda w)$. Using the subdifferential of SLOPE (Appendix \ref{appendix:slope_subdiff}) \citep{Larsson2020a}, this is given by
\begin{equation}\label{eqn:gslope_sub_condition}
   \text{cumsum}(h(\lambda)- \lambda w) \preceq \mathbf{0}.
\end{equation}
This condition is checked efficiently using the algorithm proposed for the SLOPE strong rule (Algorithm \ref{alg:slope_screen_alg}) leading to the strong rule for gSLOPE (Proposition \ref{propn:gslope_seq_strong}).
The algorithm assumes that the indices for the inactive predictors will be ordered last in the input $c$ and the features $|\hat\beta|_\downarrow$ \citep{Larsson2020a}. 
\begin{proposition}[Strong screening rule for gSLOPE]\label{propn:gslope_seq_strong}
Taking $c=h(\lambda_{k+1})$ and $\phi = \lambda_{k+1} w$ as inputs for Algorithm \ref{alg:slope_screen_alg} returns a superset $\mathcal{S}_g(\lambda_{k+1})$ of the active set $\mathcal{A}_g(\lambda_{k+1})$.
\end{proposition}
The gradient at path index $k+1$ is not available for the computation of $h(\lambda_{k+1})$, so an approximation is required that does not lead to any violations in Algorithm \ref{alg:slope_screen_alg}. By the cumsum condition in this algorithm, an approximation for a group $g\in [m]$ is sought such that $h_g(\lambda_{k+1}) \leq h_g(\lambda_{k})+ R_g$, where $R_g\geq 0$ needs to be determined. An approximation is found by assuming that $h_g(\lambda_{k+1})$ is a Lipschitz function of $\lambda_{k+1}$ with respect to the $\ell_1$ norm, that is,
\begin{equation*}
    \left|h_g(\lambda_{k+1}) -h_g(\lambda_{k}) \right| \leq w_g|\lambda_{k+1} - \lambda_k|.
\end{equation*}
By again noting that $J_\text{gslope}^*(x) = J_\text{slope}^*([x]_{\mathcal{G},-0.5})$, it can be seen that the assumption is equivalent to the Lipschitz assumptions used for the lasso and SLOPE strong rules \citep{tibshirani2010strong,Larsson2020a}. By the reverse triangle inequality, 
\begin{equation*}
  |h_g(\lambda_{k+1})|   \leq |h_g(\lambda_{k})|+\lambda_{k}w_g - \lambda_{k+1}w_g,
\end{equation*}
leading to the choice $R_g =\lambda_{k} w_g - \lambda_{k+1}w_g$ and the gradient approximation strong screening rule (Proposition \ref{propn:gslope_seq_strong_grad_approx}). To apply gSLOPE screening in practice (Section \ref{section:results_sim}), Proposition \ref{propn:gslope_path_start} describes the calculation of the first path value.


\begin{proposition}[Gradient approximation strong screening rule for gSLOPE]\label{propn:gslope_seq_strong_grad_approx} Taking $c=h(\lambda_k) + \lambda_k w - \lambda_{k+1}w $ and $\phi = \lambda_{k+1} w$ as inputs for Algorithm \ref{alg:slope_screen_alg}, and assuming that for any $k\in [l-1]$,
\begin{equation*}
      \left|h_g(\lambda_{k+1}) -h_g(\lambda_{k}) \right| \leq w_g|\lambda_{k+1} - \lambda_k|, \; \forall g=[m],
\end{equation*}
and $\mathcal{O}(h(\lambda_{k+1})) = \mathcal{O}(h(\lambda_{k}))$, then the algorithm returns a superset $\mathcal{S}_g(\lambda_{k+1})$ of the active set $\mathcal{A}_g(\lambda_{k+1})$.
\end{proposition}

\begin{proposition}[gSLOPE path start]\label{propn:gslope_path_start}
For gSLOPE, the path value at which the first group enters the model is given by
\begin{equation*}
    \lambda = \max \left\{\cumsum\bigl( ([\nabla f(\mathbf{0})]_{\mathcal{G},-0.5})_\downarrow\bigr)\oslash \cumsum(w)\right\},
\end{equation*}
where $\oslash$ denotes Hadamard division.
\end{proposition}

%% file: new_sections/4-sgs.tex
\label{sec:SGS}

This section presents the group and variable screening rules for SGS. They are derived using the SGS KKT conditions, formulated in terms of SLOPE and gSLOPE (by the sum rule of subdifferentials):
\begin{equation}\label{eqn:sgs_kkt}
 \nabla f(\beta) \in \lambda \alpha \partial J_\text{slope}(\beta; v) + \lambda (1-\alpha) \partial J_\text{gslope}(\beta; w).
\end{equation}

\subsection{Group Screening}\label{section:sgs_group_screen}
For inactive groups, the KKT conditions (Equation \ref{eqn:sgs_kkt}) for SGS are
\small{
\begin{align}
    &(\nabla f(\beta) +\lambda\alpha\partial J_\text{slope}(\mathbf{0};v))_{\mathcal{G}_\mathcal{Z}} \in  \lambda(1-\alpha)\partial J_\text{gslope}(\mathbf{0};w_\mathcal{Z}), \nonumber \\
    &\underset{\text{Equation \ref{eqn:gslope_sub_condition}}}{\implies} \text{cumsum}\Bigl( \bigl([\nabla f(\beta) +\lambda\alpha \partial J_\text{slope}(\mathbf{0};v)]_{\mathcal{G}_\mathcal{Z},-0.5}\bigr)_\downarrow \nonumber\\
    &\;\;\;\;\;\;\;\;\;\;\;\;\;\;\;\;\;\;\;\;\;\;\;\;\;\;\;\;\;\;\;\;\;\;\;\;\;\;\;\;\;\;-\lambda(1-\alpha)w_\mathcal{Z}\Bigr) \preceq \mathbf{0}.\label{eqn:group_condition_sgs}
\end{align}
}
\normalsize
The problem reduces to a form similar to the gSLOPE screening rule (Section \ref{sec:gSLOPE}), with inputs for Algorithm \ref{alg:slope_screen_alg} given by $c = ([\nabla f(\beta) + \lambda\alpha \partial J_\text{slope}(\mathbf{0};v)]_{\mathcal{G},-0.5})_\downarrow$ and $\phi = \lambda(1-\alpha)w$. 

To determine the form of the quantity $\partial J_\text{slope}(\mathbf{0};v)$, the term inside the $[\cdot]_{\mathcal{G}_\mathcal{Z},-0.5}$ operator needs to be as small as possible for Equation \ref{eqn:group_condition_sgs} to be satisfied. This term is found to be the soft thresholding operator, $S(\nabla f(\beta),\lambda\alpha v) := \text{sign}(\nabla f(\beta))(|\nabla f(\beta)| - \lambda\alpha v)_+$ by Lemma \ref{lemma:sto} (see Appendix \ref{appendix:sgs_sto} for the proof).
\begin{lemma}\label{lemma:sto}
   In Equation \ref{eqn:group_condition_sgs}, choosing $\partial J_\text{slope}(\mathbf{0};v)=S(\nabla f(\beta),\lambda\alpha v)$ minimises $[\nabla f(\beta) +\lambda\alpha \partial J_\text{slope}(\mathbf{0};v)]_{\mathcal{G}}$.
\end{lemma}
By using the soft-thresholding operator, a valuable connection between SGS and SGL is found, as the operator is used in the gradient update step for SGL \citep{Simon2013}. This connection has the potential to lead to new and more efficient optimization approaches for SGS that are more closely related to those used to solve SGL, similar to the recent coordinate descent algorithm for SLOPE \citep{Larsson2022CoordinateSLOPE}.

Using this operator, the (non-approximated) strong group screening rule for SGS is shown in Proposition \ref{propn:sgs_screen}. Applying a similar Lipschitz assumption as for the gSLOPE rule 
gives the gradient approximation strong group screening rule for SGS (Proposition \ref{propn:sgs_screen_grad_approx}).

\begin{proposition}[Gradient approximation strong group screening rule for SGS]\label{propn:sgs_screen_grad_approx}
Let $\tilde{h}(\lambda) := ([S(\nabla f(\hat\beta(\lambda)),\lambda\alpha v)]_{\mathcal{G},-0.5})_\downarrow$. Taking $c=\tilde{h}(\lambda_{k}) +  \lambda_k(1-\alpha)w  -\lambda_{k+1}(1-\alpha) w $ and $\phi = \lambda_{k+1}(1-\alpha) w$ as inputs for Algorithm \ref{alg:slope_screen_alg}, and assuming that for any $k\in [l-1]$,
    \begin{equation*}
\left|\tilde{h}_g(\lambda_{k+1}) -\tilde{h}_g(\lambda_{k}) \right| \leq (1-\alpha)w_g|\lambda_{k+1} - \lambda_k|, \; \forall g=[m], 
    \end{equation*}
    and $\mathcal{O}(\tilde{h}(\lambda_{k+1})) = \mathcal{O}(\tilde{h}(\lambda_{k}))$, then the algorithm returns a superset $\mathcal{S}_g(\lambda_{k+1})$ of the active set $\mathcal{A}_g(\lambda_{k+1})$.
\end{proposition}

\subsection{Variable Screening}
By exploiting the sparse-group norm of SGS, the input dimensionality can be reduced further with a second layer of variable screening. The KKT conditions (Equation \ref{eqn:sgs_kkt}) for zero variables in active groups are
\begin{equation}\label{eqn:sgs_non_zero_grp}
    \nabla_{\mathcal{G}_{\mathcal{A}_g}} f(\beta) \in \lambda \alpha\partial J_\text{slope}(\mathbf{0};v_{\mathcal{G}_{\mathcal{A}_g}}). 
\end{equation}
The gSLOPE subdifferential term vanishes as the numerator is zero in Theorem \ref{thm:gslope_subdiff}. The problem reduces to that of SLOPE screening, applied to the variables in groups in $\mathcal{A}_g$ and scaled by $\alpha$. The gradient approximated rule is shown in Proposition \ref{propn:sgs_screen_var_grad_approx} (see Proposition \ref{propn:sgs_screen_var} for the non-approximated version).
\begin{proposition}[Gradient approximation
strong variable screening rule for SGS]\label{propn:sgs_screen_var_grad_approx}
Let $\bar{h}(\lambda) := (\nabla f(\hat\beta(\lambda)))_{|\downarrow|}$. Taking $c = |\bar{h}(\lambda_{k+1})| + \lambda_k\alpha v - \lambda_{k+1}\alpha v$ and $\phi = \lambda_{k+1}\alpha v$ for the variables in the groups in $\mathcal{A}_g(\lambda_{k+1})$ as inputs for Algorithm \ref{alg:slope_screen_alg}, and assuming that for any $k\in [l-1]$,
    \begin{equation*}
    \left|\bar{h}_j(\lambda_{k+1}) -\bar{h}_j(\lambda_{k}) \right| \leq \alpha v_j|\lambda_{k+1} - \lambda_k|, \forall j\in \mathcal{G}_{\mathcal{A}_g(\lambda_{k+1})},
    \end{equation*}
  and $\mathcal{O}(\bar{h}(\lambda_{k+1})) = \mathcal{O}(\bar{h}(\lambda_{k}))$, then the algorithm returns a superset $\mathcal{S}_v(\lambda_{k+1})$ of $\mathcal{A}_v(\lambda_{k+1})$.
\end{proposition}
In practice $\mathcal{A}_g(\lambda_{k+1})$ is not available, as this is exactly what we are trying to superset with any screening rule. However, Proposition \ref{propn:sgs_screen_grad_approx} guarantees that it is contained in $\mathcal{S}_g(\lambda_{k+1})$ so that this can be used instead. To apply SGS in practice (Section \ref{section:results_sim}), Proposition \ref{propn:sgs_path_start} describes the calculation of the first path value.

\begin{proposition}[SGS path start]\label{propn:sgs_path_start}
For SGS, the path value at which the first variable enters the model is
\begin{align*}
    \lambda = \max\{\cumsum(| \nabla &f(\mathbf{0})|_\downarrow)\oslash\\
    &\cumsum((1-\alpha)\tau\omega- \alpha v)\},
\end{align*}
where $\tau$ and $\omega$ are expanded vectors of the group sizes ($\sqrt{p_g})$ and penalty weights ($w_g$) to $p$ dimensions. 
    
\end{proposition}
\section{GROUP-BASED OWL}\label{section:owl}
The screening rule framework presented are also applicable to the wider class of OWL models. The Ordered Weighted $\ell_1$ (OWL) framework is defined as \citep{Zeng2014DecreasingWS}
\begin{equation*}
    \hat\beta \in \argmin_{\beta \in \mathbb{R}^p} \left\{\nabla f(\beta) + \lambda J_\text{owl}(\beta;v)\right\},
\end{equation*}
where $J_\text{owl}(\beta;v) =\sum_{i=1}^p v_i |\beta|_{(i)}$, $|\beta|_{(1)} \geq \ldots \geq |\beta|_{(p)}$, and $v$ are non-negative non-increasing weights. SLOPE is a special case of OWL where the weights are taken to be the Benjamini-Hochberg critical values \citep{Bogdan2015SLOPEAdaptiveOptimization}. Octagonal Shrinkage and Clustering Algorithm for Regression (OSCAR) \citep{oscar} is a further special case of OWL (often referred to as OWL with linear decay) where for a variable $i\in[p]$, the weights are taken to be $v_i = \sigma_1 + \sigma_2(p-i)$, and $\sigma_1, \sigma_2$ are to be set. In \cite{Bao2020FastRules} they are set to $\sigma_1 = d_i\|\mathbf{X}^\top y\|_\infty, \sigma_2 = \sigma_1/p$, where $d_i = i\times e^{-2}$, $\mathbf{X}^\top\in\mathbb{R}^{n\times p}$ is the design matrix, and $y\in\mathbb{R}^n$ is the response vector.

Group OSCAR (gOSCAR) and Sparse-group OSCAR (SGO) are defined using the frameworks provided by gSLOPE \citep{Brzyski2019GroupPredictors} and SGS \citep{Feser2023Sparse-groupFDR-control}, respectively, but instead use the OSCAR weights (see Appendix \ref{appendix:oscar_weights}). 

%% file: new_sections/5-synthetic.tex
This section illustrates the effectiveness of the screening rules for gSLOPE and SGS using synthetic (Section \ref{section:results_sim}) and real (Section \ref{section:results_real}) data. References to $\mathcal{E},\mathcal{A}$ for group and variable metrics denote $\mathcal{E}_g,\mathcal{A}_g$ and $\mathcal{E}_v,\mathcal{A}_v$. For SGS, $\mathcal{E}_g$ represents groups with members in $\mathcal{E}_v$. 

\subsection{Synthetic Data Analysis} \label{section:results_sim}
\paragraph{Set up} A multivariate Gaussian design matrix, $\mathbf{X}\sim\mathcal{N}(\mathbf{0},\boldsymbol{\Sigma})\in \mathbb{R}^{400 \times p}$, was used and the within-group correlation set to $\Sigma_{i,j} = \rho$, where $i$ and $j$ belong to the same group. The correlation and number of features were varied between $\rho \in \{0,0.3,0.6,0.9\}$ and $p \in \{500,1625,2750,3875,5000\}$, producing $20$ simulation cases. Each simulation case was repeated $100$ times. Two models were considered: linear and logistic. For the linear model, the output was generated as $y = \mathbf{X}\beta + \mathcal{N}(0,1)$ and for the logistic model the class probabilities were calculated using $\sigma(\mathbf{X}\beta + \mathcal{N}(0,1))$, where $\sigma$ is the sigmoid function. Groups of sizes between $3$ and $25$ were considered, of which $15\%$ were set to active. Within each active group, $30\%$ of the variables were set to active with signal $\beta \sim \mathcal{N}(0,5)$.
\begin{figure}[t]
\centering
  \includegraphics[width=1\linewidth]{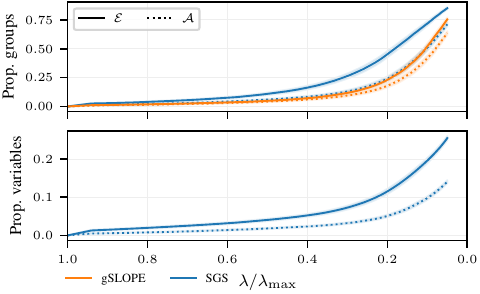}
  \caption{The proportion of groups/variables in $\mathcal{E}, \mathcal{A}$, relative to the input, for both gSLOPE and SGS as a function of the path for the linear model with $p=2750, \rho=0.6, m = 197$. The results are averaged over $100$ repetitions, with $95\%$ confidence intervals shown.}
   \label{fig:path_plot_gslope_sgs}
\end{figure}
gSLOPE and SGS were fit along a log-linear path of $50$ regularization parameters using warm starts, beginning at $\lambda_1$ (from Propositions \ref{propn:gslope_path_start} and \ref{propn:sgs_path_start}), and ending at $\lambda_{50} = 0.05\lambda_1$. The data was $\ell_2$ standardized and for the linear model an intercept was used. Both models had FDR-control parameters set to $0.05$, and $\alpha = 0.95$ for SGS. The models were fit using the adaptive three operator splitting (ATOS) algorithm \citep{pmlr-v80-pedregosa18a}, although the screening rules can be applied with any fitting algorithm.  Additional computational details are in Appendix \ref{appendix:comp_info}. 

Primarily, the results for the linear model are presented, and the results for the logistic model are in Appendix \ref{appendix:log_model_plot}. The simulations were repeated for group-based OSCAR models (Appendix \ref{appendix:oscar}).



\begin{table}[b]
  \caption{Runtime (in seconds) for fitting $50$ models along a path, shown for screening against no screening, for the linear and logistic models. The results are averaged across all cases of the correlation ($\rho$) and dimensionality ($p$), with standard errors shown.}
\label{tbl:sim_data}
  \begin{tabular}{llll}
    \toprule
    Method     & Type     & Screen (s) & No screen (s) \\
    \midrule
gSLOPE& Linear&$1016\pm21$ & $1623\pm27$   \\
gSLOPE &Logistic& $814\pm8$ & $1409\pm11$\\
SGS&Linear&  $735\pm15$ & $1830\pm34$ \\
SGS& Logistic& $407\pm2$& $859\pm6$\\
    \bottomrule
  \end{tabular}
  
\end{table}
\paragraph{Screening Efficiency} 
By comparing the sizes of the fitting set $(\mathcal{E})$ to the active set ($\mathcal{A})$, the screening rules are found to be efficient in providing dimensionality reduction close to the minimum possible (the active set size) (Figures \ref{fig:path_plot_gslope_sgs} and \ref{fig:path_plot_gslope_sgs_metric_2}). As expected, the sets increase in size as $\lambda$ decreases, and the difference in size between the sets remains stable along the path, decreasing towards the termination point. The size of the fitting set remains far below the input size across the whole path, showing the benefit of the screening. This is found to be true for any correlation, input dimensionality, and model considered (Appendix \ref{appendix:results_sim_study}). 

The screening rules perform well for linear and logistic models (Figure \ref{fig:gaussian_vs_log}), showing robust dimensionality reduction for all correlation cases considered. As the correlation increases, the signal concentrates in fewer groups, causing the active group set to decrease in size. SLOPE models deal well with highly correlated features, as the sorted norm clusters them together \citep{Zeng2014TheAlgorithm}.

The screening rules are found to efficiently reduce the input dimensionality on average across all cases considered (Table \ref{tbl:full_data}).

\begin{figure}[t]
  \centering
  \includegraphics[width=\linewidth]{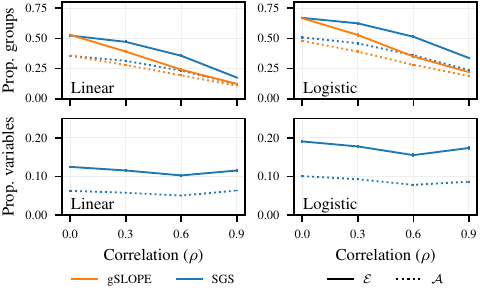}
  \caption{The proportion of groups/variables in $\mathcal{E}, \mathcal{A}$, relative to the full input, shown for gSLOPE and SGS. This is shown as a function of the correlation ($\rho$), averaged over all cases of the input dimension ($p$), with $100$ repetitions for each $p$, for both linear and logistic models, with standard errors shown.}
      \label{fig:gaussian_vs_log}
\end{figure}
\paragraph{Runtime Performance} A key metric of performance for a screening rule is the time taken to fit a path of models. Figure \ref{fig:time_fcn_of_p} shows the significant runtime improvements our screening rules provide as a function of increasing input dimensionality. The gain is observed to be substantial under all correlation cases. Applying screening improves the scaleability of applying gSLOPE and SGS models to larger datasets.

The clear benefit and robustness of our screening method, with regards to runtime, can be seen by aggregating the results of all the simulation cases (Table \ref{tbl:sim_data}). For both models, in the linear and logistic cases, screening substantially improves the computational cost, halving the runtime for SGS models.

\subsection{Real Data Experiments}\label{section:results_real}
\paragraph{Datasets} The screening rules were applied to seven real gene datasets, of different response types and dimensionality. Two of the datasets, carbotax and sheetz, had a continuous response so were fit using a linear model. For these, the groups were generated using K-means clustering \citep{1056489}. The remaining five (adenoma, cancer, celiac, colitis, and tumour), had binary labels, so a logistic model was used. For these, the design matrices contained gene expression data downloaded from NCBI's GEO database \citep{Edgar2002}, so the genes could be assigned to pathways (groups) using the C3 regulatory target gene sets from MSigDB \citep{Subramanian2005,msigdb}. All datasets were high-dimensional. See Appendix \ref{appendix:real_data_description} for full details.


Both gSLOPE and SGS were fit with their FDR-control parameters set to $0.01$ and for SGS $\alpha=0.99$. Each model was applied along a path of $100$ regularization parameters, with $\lambda_{100} = 0.01\lambda_1$. Table \ref{tbl:atos_params} describes the algorithmic parameters used for ATOS.

\begin{figure}[t]
  \centering
\includegraphics[width=1\linewidth,valign=t]{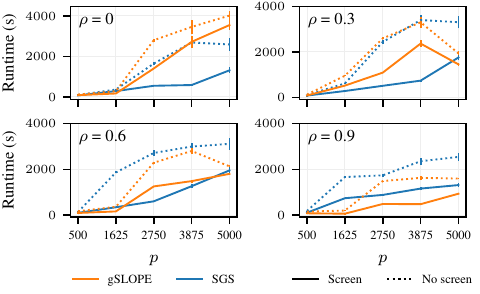}
  \caption{Runtime (in seconds) for fitting $50$ models along a path, shown for screening against no screening as a function of $p$, broken down into different correlation cases, for the linear model. The results are averaged over $100$ repetitions, with standard errors shown.}
\label{fig:time_fcn_of_p}
\end{figure}

\paragraph{Results} For both gSLOPE and SGS, the screening rules led to considerably faster runtimes for all datasets (Figure \ref{fig:real_bar_chart}). The screening was more effective for SGS under a linear model and for gSLOPE under a logistic model. The active set sizes for gSLOPE tended to be smaller for the logistic models, allowing for larger feature reduction (Table \ref{tbl:full_data}). For SGS, this trend went the other way; the active sets were smaller for the continuous datasets. Another explanation for the reduced screening efficiency for gSLOPE under linear models was due to the grouping structure. For the continuous responses, the groups were generated using K-means clustering, leading to fewer groups and less opportunity for dimensionality reduction (Table \ref{tbl:appendix_real_dataset_info}). With fewer groups, gSLOPE is less likely to discard full groups, as they may contain some signal.

The feature reduction provided by our screening rules led to the alleviation of convergence issues (Table \ref{tbl:full_table_grp_real_data}). SGS failed to converge for three datasets without screening and none with screening. gSLOPE experienced failed convergences across all datasets without screening, while with screening, it only failed to converge for three datasets, each showing fewer instances of failure. As gSLOPE applies no variable penalization, it is forced to fit all variables within a group. For datasets with large groups, such as those considered here, this leads to a problematic fitting process which can include many noisy variables. Our screening rules help gSLOPE partially overcome this issue, leading to large computational savings and better solution optimality.

The analysis of the real data further illustrates the benefits of the bi-level screening to the runtime and performance of SGS (for synthetic data, see Figure \ref{fig:bi-level-screening}). Figure \ref{fig:bilevelcomparison} illustrates that for the cancer and celiac datasets, the bi-level screening allows the input dimensionality for SGS to be reduced to a much greater extent than by just group screening (see Figure \ref{fig:bilevelcomparison_all} for the other datasets). 

Both screening rules drastically reduce the input dimensionality on all the real datasets (Table \ref{tbl:full_data}). On average, for the logistic real data models, the SGS screening rules reduced the input to just $7\%$ of the total space. This comes without affecting solution consistency (Appendices \ref{appendix:comp_real_info} and \ref{appendix:real_data_results}).
\begin{figure}[t]
 \centering
\includegraphics[width=1\linewidth,valign=t]{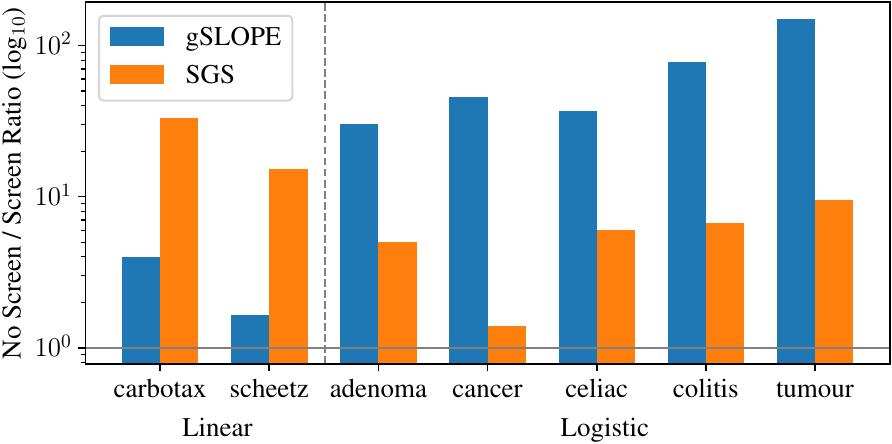}
  \captionsetup{width=1\linewidth}
  \caption{The ratio of no screen time to screen time ($\uparrow$) of gSLOPE and SGS applied to the real datasets, for fitting $100$ path models, split into response type. The horizontal grey line represents no screening improvement.}
  \label{fig:real_bar_chart}
\end{figure}
\begin{table}[b]
  \caption{The cardinality of the active ($\mathcal{A}$) and fitting ($\mathcal{E}$) sets for gSLOPE and SGS, averaged across all synthetic and real data cases, split into model type. For gSLOPE, the cardinality is for the group sets, and for SGS the variable ones. Dim. is the input dimensionality. For synthetic, it was $p = 2750$ and $m =210$.}
\label{tbl:full_data}
  \resizebox{\linewidth}{!}{%
  \begin{tabular}{lllllll}
    \toprule
    && \multicolumn{2}{c}{Synthetic}  & \multicolumn{3}{c}{Real}  \\
    \cmidrule(lr){3-4}  \cmidrule(lr){5-7} 
    Method     & Type     & $\mathcal{A}$ & $\mathcal{E}$ & $\mathcal{A}$ & $\mathcal{E}$&Dim. \\
    \midrule
gSLOPE & Lin. & $55 \scriptstyle{\pm 1}$ & $76 \scriptstyle{\pm 1}$  & $112 \scriptstyle{\pm 3}$ & $168 \scriptstyle{\pm 3}$ & $240$ \\ 
gSLOPE & Log. & $71 \scriptstyle{\pm 1}$ & $97 \scriptstyle{\pm 1}$  & $49 \scriptstyle{\pm 2}$ & $83 \scriptstyle{\pm 3}$ & $1634$ \\ 
SGS    & Lin. & $178 \scriptstyle{\pm 3}$ & $364 \scriptstyle{\pm 6}$  & $378 \scriptstyle{\pm 37}$ & $1139 \scriptstyle{\pm 87}$ & $14488$ \\ 
SGS    & Log. & $230 \scriptstyle{\pm 3}$ & $472 \scriptstyle{\pm 6}$  & $526 \scriptstyle{\pm 5}$ & $979 \scriptstyle{\pm 13}$ & $13734$ \\ 
    \bottomrule
  \end{tabular}
  }
\end{table}
For additional metrics and comparison to the SLOPE strong rule for the real datasets, see Figures \ref{fig:real_bar_chart_slope}, \ref{fig:real_data_propn}, and \ref{fig:real_data_propn_active}. Our rules for gSLOPE and SGS are found to offer similar levels of runtime improvements and feature reduction to that of SLOPE \citep{Larsson2020a}.
\begin{figure}[t]
    \centering
    \includegraphics[width=1\linewidth]{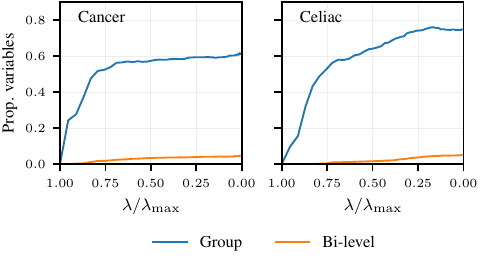}
    \captionsetup{width=1\linewidth}
    \caption{The proportion of variables in $\mathcal{S}_v$ relative to the full input for group-only and bi-level screening applied to SGS, plotted along the regularization path for the cancer and celiac datasets.}
\label{fig:bilevelcomparison}
\end{figure}

\subsection{KKT Violations}\label{section:kkt_violations}
As with any strong screening rule, our approach depends on assumptions that may fail. When this happens, KKT checks are used to ensure no active variables are excluded, with violations added to $\mathcal{E}_v$. For SGS, KKT violations occur at a variable-level (Appendix \ref{appendix:sgs_kkt}) and for gSLOPE at a group-level (Appendix \ref{appendix:gslope_kkt}).

KKT violations are very rare for gSLOPE (Figure \ref{fig:kkt_spikeplot_synth}), occurring on the simulated data infrequently toward the start of the path. On the real data, only a single dataset had violations, and the number of violations was very small (Table \ref{tbl:full_table_grp_real_data}).

For SGS, KKT violations are more common, but still infrequent (Figure \ref{fig:kkt_spikeplot_synth} and Tables \ref{tbl:full_table_var_real_data}) due to additional assumptions in the second layer of screening. In Equation \ref{eqn:group_condition_sgs}, minimizing the subdifferential term leads to tighter screened sets, contributing to these violations. In Figure \ref{fig:kkt_spikeplot_synth}, the number of violations increases as a function of the model density for SGS, mirroring the log-linear shape of the regularization path. This pattern is also seen in the strong rule for SLOPE \citep{Larsson2020a}.

Despite these violations and the computational cost of the checks and refitting (Algorithm \ref{alg:sgs_framework}), the overall screening process yields significant runtime improvements in all cases (Table \ref{tbl:sim_data} and Figure \ref{fig:real_bar_chart}).

\begin{figure}[t]
    \centering
    \includegraphics[width=1\linewidth]{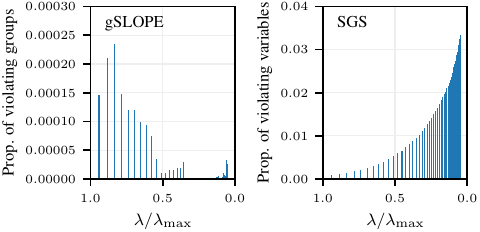}
      \captionsetup{width=1\linewidth}
    \caption{The proportion of KKT violations relative to the full input for gSLOPE and SGS, under linear models, averaged over all synthetic data cases. This is shown as a function of the regularization path.}
    \label{fig:kkt_spikeplot_synth}
\end{figure}

%% file: new_sections/6-discussion.tex
In this manuscript, we have developed strong screening rules for group-based SLOPE models using our new sparse-group strong screening framework: Group SLOPE and Sparse-group SLOPE, neither of which have any previous screening rules. Our proposed screening rules differ from the existing SLOPE strong rule both in construction and in outcome. The screening rule for gSLOPE screens out irrelevant groups before fitting. The screening rules for SGS perform bi-level screening. Our rules apply to the wider class of OWL models, including group-based OSCAR models.

SLOPE models are finding increasing use in genetics and machine learning, with SGS found to have superior disease prediction performance over other penalized methods \citep{Feser2023Sparse-groupFDR-control}. Our screening rules will make the group-based versions more accessible by reducing their computational burden. This will allow practitioners to utilize their FDR properties, facilitating their widespread use across various fields. 


Through comprehensive analysis of synthetic and real data, we illustrate that the screening rules lead to dramatic improvements in the runtime of gSLOPE and SGS models, as well as for group-based OSCAR models (Appendix \ref{appendix:oscar}). This is achieved without affecting model accuracy. This is particularly important in datasets where $p\gg n$, such as genetics ones, which is the main motivation behind the proposal of SLOPE \citep{Bogdan2015SLOPEAdaptiveOptimization}. The screening rules presented in this manuscript allow group-based SLOPE, and more generally group-based OWL models, to achieve computational fitting times more in line with their lasso-based counterparts. The screening rules also helped gSLOPE and SGS overcome convergence issues in large datasets, improving solution optimality.

In our data studies, we have not discovered any scenario where our screening rules did not perform better than no screening. In each case, the rules greatly reduce the input dimensionality and speed up the computational runtime.



\paragraph{Limitations} 

Our screening rules, as any strong rule, rely on assumptions. For both gSLOPE and SGS, Lipschitz assumptions were used that are consistent with those used in the strong screening framework \citep{tibshirani2010strong}. For gSLOPE, a Lipschitz assumption was used to derive Proposition \ref{propn:gslope_seq_strong_grad_approx}, while for SGS, a separate Lipschitz assumption was made for each layer of screening (Propositions \ref{propn:sgs_screen_grad_approx} and \ref{propn:sgs_screen_var_grad_approx}).

Violations of these assumptions are checked for using the KKT conditions. The SGS KKT checks (Appendix \ref{appendix:sgs_kkt}) led to an increased number of violations, as the checks were overly conservative. However, the overall amount of KKT violations for both models was still relatively small, suggesting that these assumptions are not overly restrictive.

An attempt was made to derive less conservative checks (Appendix \ref{appendix:alternative_checks}), performed directly on the variables without an initial group check. These were found to be too lenient, incorrectly missing violations. Future work can consider alternative KKT checks.


An additional future direction includes the development of safe rules, which guarantee that only inactive variables are discarded. These could be incorporated into a hybrid scheme together with our strong rules \citep{Zeng2021,Wang2022}. Deriving safe rules would facilitate a comparison between them and our proposed strong rules, offering further insight into which type of screening is most effective for both non-separable and sparse-group norms.

\subsection*{Acknowledgements}
We would like to thank the anonymous reviewers for
their valuable comments. This work was supported by the Engineering and Physical Sciences Research Council (EPSRC) through the Modern Statistics and Statistical Machine Learning (StatML) CDT programme, grant no. EP/S023151/1.